\documentclass{article} % For LaTeX2e
\usepackage{iclr2021_conference,times}

\usepackage{amsmath,amsfonts,bm}

\def\eqref#1{equation~\ref{#1}}
\def\1{\bm{1}}

\DeclareMathAlphabet{\mathsfit}{\encodingdefault}{\sfdefault}{m}{sl}
\SetMathAlphabet{\mathsfit}{bold}{\encodingdefault}{\sfdefault}{bx}{n}

\usepackage{hyperref}
\usepackage{url}

\usepackage[utf8]{inputenc} % allow utf-8 input
\usepackage[T1]{fontenc}    % use 8-bit T1 fonts
\usepackage{hyperref}       % hyperlinks
\usepackage{url}            % simple URL typesetting

\usepackage{booktabs}       % professional-quality tables
\usepackage{amsfonts}       % blackboard math symbols
\usepackage{nicefrac}       % compact symbols for 1/2, etc.
\usepackage{microtype}      % microtypography
\usepackage{natbib}
\usepackage{amsmath}
\usepackage{xcolor}
\usepackage{tikz}
\usepackage{wrapfig}
\usepackage{graphicx}
\usepackage{amsthm}
\usepackage{subcaption}
\usepackage{mathtools}
\usepackage{multirow}
\usepackage{amsthm}
\usepackage{amsmath}
\usepackage{graphicx}
\usepackage{algorithm}
\usepackage[noend]{algpseudocode}
\usepackage{wrapfig}
\usepackage{subcaption}
\usepackage{placeins}
\graphicspath{ {./figures/} }
\newtheorem{theorem}{Theorem}[section]

\newtheorem{lemma}[theorem]{Lemma}
\newtheorem{proposition}[theorem]{Proposition}
\newtheorem{remark}[theorem]{Remark}

\newtheorem{assumption}[theorem]{Assumption}

\title{Convergence Analysis of Homotopy-SGD for non-convex optimization}

\author{Matilde Gargiani$^1$, Andrea Zanelli$^{2}$, Quoc Tran-Dinh$^3$, Moritz Diehl$^{2,4}$, Frank Hutter$^{1,5}$\\ %\textbf{$^1$, Frank Hutter$^{1,2}$} \\
$^1$Department of Computer Science, University of Freiburg\\
\texttt{\{gargiani, fh\}@cs.uni-freiburg.de} \\
$^2$Department of Microsystems Engineering (IMTEK), University of Freiburg\\% (IMTEK), University of Freiburg\\
\texttt{\{andrea.zanelli, moritz.diehl\}@imtek.uni-freiburg.de} \\
$^3$Department of Statistics and Operations Research, University of North Carolina\\
\texttt{quoctd@email.unc.edu} \\
$^4$Department of Mathematics, University of Freiburg\\
$^5$Bosch Center for Artificial Intelligence
}

\iclrfinalcopy % Uncomment for camera-ready version, but NOT for submission.
\begin{document}

\maketitle

\begin{abstract}
First-order stochastic methods for solving large-scale non-convex optimization problems are widely used in many big-data applications, e.g. training deep neural networks as well as other complex and potentially non-convex machine learning models.
%~\citep{NIPS2003_2365, a4bb83da67734d249d9fc61e84d4eb2c}.   	
	Their inexpensive iterations generally come together with slow global convergence rate (mostly sublinear), leading to the necessity of carrying out a very high number of iterations before the iterates reach a neighborhood of a minimizer.  
  In this work, we present a first-order stochastic algorithm based on a combination of homotopy methods and SGD, called Homotopy-Stochastic Gradient Descent (H-SGD), which finds interesting connections with some proposed heuristics in the literature, e.g. optimization by Gaussian continuation, training by diffusion, mollifying networks. 
  %~\citep{AAAI159299, 2016arXiv160104114M, DBLP:conf/iclr/GulcehreMVB17}. %The core idea consists in combining homotopy methods with SGD.   
  Under some mild assumptions on the problem structure, we conduct a theoretical analysis of the proposed algorithm. Our analysis shows that, with a specifically designed scheme for the homotopy parameter, H-SGD enjoys a global linear rate of convergence to a neighborhood of a minimum while maintaining fast and inexpensive iterations. %In particular, we show that by adopting an exponential scheme for the homotopy-parameter, we are able to preserve linear %In particular, we show that with a linear scheme for the homotopy-parameter, H-SGD tracks in expectation an $r$-optimal solution from the source to the target homotopy problem, while with an exponential scheme the method achieves global linear rate of convergence. %first we derive the conditions on the homotopy-step parameter for the optimality tracking of H-SGD. Then, we derive conditions for linear convergence rate of H-SGD. 
  Experimental evaluations %on a one dimensional toy problem and on a neural network regression and on a binary classification problem 
  confirm the theoretical results and show that H-SGD can outperform standard SGD. %the state-of-the-art stochastic first-order methods, i.e. SGD~
  \end{abstract}

\section{Introduction}
This paper focuses on the theoretical development and analysis of a stochastic optimization algorithm, called Homotopy-Stochastic Gradient Descent (H-SGD), based on the combination of homotopy methods and stochastic gradient descent (SGD). The algorithm we propose is specifically designed to solve finite-sum problems of the following form
\begin{equation}\label{eq: main_problem}
w^* \in \arg\min_{w\in\mathbb{R}^d} \,\left\{f(w)\coloneqq \frac{1}{N}\sum_{j=1}^N\, f_j(w)\right\},\,
\end{equation} 
where $f:\mathbb{R}^d\rightarrow \mathbb{R}$ is continuously differentiable, bounded below and not necessarily convex.
In particular, we assume that we only have access to noisy function values and gradients of the objective function in~\eqref{eq: main_problem} via a stochastic first-order oracle, as in~\citep{10.1137/070704277} and~\citep{Ghadimi2013StochasticFA}. Problems of this form typically arise in machine learning and deep learning applications, where the dimensionality of the datasets makes the full function and gradient evaluations too expensive. This class of problems is generally approximately solved by stochastic first-order iterative algorithms, e.g. SGD~\citep{a4bb83da67734d249d9fc61e84d4eb2c}, Adagrad~\citep{JMLR:v12/duchi11a}, Adam~\citep{DBLP:journals/corr/KingmaB14}. 
At the iteration $t$, the algorithms of this class acquire a stochastic estimate of the function value $f(w_t,\,\xi_t)$ and the gradient $g(w_t,\,\xi_t)$ by calling the oracle with input $w_t$, where $\xi_t$ is a random variable, i.e. when the noise comes from subsampling as in the mini-batch scenarios, then $\xi_t\in \left\{0,1 \right\}^N$ with ${\Vert \xi_t \Vert}_1 =M$ and $g(w_t,\,\xi_t)= \frac{1}{M}\sum_{j=1}^{N}\xi_{t,j}\cdot\nabla f_j(w_t)$. % whose distribution $P_t$ has support $\Sigma_t \subseteq \mathbb{R}^d$. 
In the case of SGD, for a given $w_0\in\mathbb{R}^d$ and $\alpha>0$, the iterates are generated as follows
\begin{equation}\label{eq:main_sgd_iter}
w_{t+1} \coloneqq w_{t} - \alpha g(w_t,\,\xi_t)\,.
\end{equation}
Consequently, the iterate $w_{t+1}=w_{t+1}(\xi_{[t]})$ is a function of the history $\xi_{[t]}\coloneqq \left(\xi_0,\,\dots,\,\xi_{t}\right)$ (also $w_0$ should be included in case it is a random initial point) of the generated random process and hence is itself random. 

%motivation
In general, stochastic first-order methods enjoy fast convergence when the problem is characterized by a certain structure. In particular, when the Polyak-Łojasiewicz (PL) condition~\citep[see][for more details on the PL condition]{10.1007/978-3-319-46128-1_50} holds for the objective function in Problem~\ref{eq: main_problem}, then, with a ``small enough'' value for the step-size, the SGD iterates converge linearly to a minimizer's neighborhood~\citep{10.1007/978-3-319-46128-1_50, pmlr-v89-vaswani19a}. Unfortunately, in many machine learning applications, the PL condition is not a realistic assumption as the landscape is generally characterized by the presence of multiple local minima and saddle points~\citep{NIPS2014_5486, NIPS2016_6112, PL_lecture}. At the same time, in the vicinity of the minimizers the problems generally show stronger structures, i.e. PL or even strong convexity, allowing for a faster local convergence~\citep{PL_lecture}. In such a scenario, a smart initialization hence becomes crucial for the numerical performance of the method~\citep{pmlr-v28-sutskever13}. Unfortunately, the power of the existing smart initialization heuristics is often quite limited given the small knowledge of the problem's landscape which we generally dispose of. In addition, these heuristics typically can not guarantee that the SGD iterates start ``close enough'' to a minimizer, i.e. in the region where the PL condition holds, such that the method enjoys a linear rate of convergence. Therefore, the ideal scenario would be to be able to exploit the stronger local structure while the method's iterates gradually approach a minimizer and independently from the starting point. %%This might sound quite unrealistic unless the initialization allows the method to start ``close enough'' to a minimizer, which can not be guaranteed generally given the limited knowledge of the problem landscape. 
%add statement to say that we want to exploit the local structures more as in these regions we can enjoy faster convergence of SGD
In this regard, homotopy methods are a general strategy for tackling difficult optimization problems by gradually transforming a simplified version of a target problem, or a version with a known minimizer, back to its original form while following a solution along the way. Consequently, they preserve in each step the vicinity to a minimizer of the currently tackled problem, allowing the solver to always work in regions where the problems exhibit stronger structures.
In general terms, homotopy methods~\citep{10.5555/945750} are a widely and successfully used mathematical tool to efficiently solve various problems in numerical analysis, e.g.~\citep{10.5555/2103574},~\citep{HAM_liao}. %The core idea consists in sequentially solving a series of parametric problems, starting from an easy-to-solve problem and progressively deforming it, via a homotopy function, to the target one. 
Such methods are also suitable to solve complex non-convex optimization problems where no or only little prior knowledge regarding the localization of the solutions is available, allowing for the exploitation of the stronger local structures of the problems in order to achieve fast global convergence, e.g.~\citep{H_proximal_L1, pmlr-v32-lin14, pmlr-v32-suzumura14, Gargiani2020TransferringOA}.%~\citep{pmlr-v65-anandkumar17a}, % and, in contrast to state-of-the-art algorithms in deep learning, they allow for the exploitation of the stronger local structures of the problem in order to achieve fast global convergence, i.e.~\citep{H_proximal_L1},~\citep{pmlr-v32-lin14},~\citep{pmlr-v32-suzumura14},~\citep{Gargiani2020TransferringOA},~\citep{pmlr-v65-anandkumar17a}.

In this work, we propose a stochastic first-order numerical method to solve Problem~\ref{eq: main_problem}, called Homotopy-Stochastic Gradient Descent (H-SGD), which is based on the combination of the homotopy method and SGD. After introducing the method and discussing the related work (Section~\ref{sec:method_related_works}), our contributions are as follows
\begin{enumerate}
\item In Section~\ref{sec:theoretical_analysis}, we provide a general theoretical analysis of H-SGD under some mild assumptions, showing that, if the increments in the homotopy parameter are ``small enough'', the proposed method tracks in expectation an $r$-optimal solution across homotopy iterations. We then show that, in the same setting, H-SGD can achieve a global linear rate of convergence to a minimizer's neighborhood when used in combination with a specific schedule for the homotopy parameter, i.e. $\Delta\lambda_i$ decreases exponentially across homotopy iterations.
\item In Section~\ref{sec:experiments}, we empirically evaluate the performance of H-SGD. Our experiments not only confirm the theoretical results derived in Section~\ref{sec:theoretical_analysis} but also show that H-SGD with a smartly designed homotopy map can outperform SGD.
\end{enumerate}

\section{Homotopy-SGD}\label{sec:method_related_works}

%homotopy sgd
Homotopy-Stochastic Gradient Descent (H-SGD) is based on the combination of the homotopy method and SGD, in the hope of combining the best of both worlds.  
%cheap iterations with a cost that is independent from N
In particular, the goal is that of maintaining the advantageous properties of SGD, such as its cheap iterations and fast local convergence under PL condition, while maximally exploiting the stronger local structures via the homotopy scheme. Therefore, H-SGD relies on the definition of a homotopy map $f(w,\,\lambda):\mathbb{R}^d\times [0,1]\rightarrow \mathbb{R}$, such that, when $\lambda=0$ we recover a well-behaved function, e.g. convex, or a function with a known minimizer's localization, and by increasing the $\lambda$ parameter, also called homotopy parameter, we gradually
%increment its complexity, e.g. non-convexity, or simply shift it
morph it in order to finally end up with our target objective function $f(w,\,1)=f(w)$~\citep[see][for more details on homotopy functions]{suciu_lecture}. By using such a homotopy map, H-SGD finds an approximate solution of Problem~\ref{eq: main_problem} by approximately solving a series of parametric problems that gradually leads to the target one. In particular, in each homotopy iteration $i$, H-SGD tackles a parametric problem of the form
\begin{equation}
w^*_{i}\in \arg\min_{w\in\mathbb{R}^d} \, f(w,\,\lambda_i)\,,
\end{equation}
where the homotopy parameter $\lambda_i$ is slightly increased at each homotopy iteration.
  As confirmed by our theoretical analysis, if the variations of the homotopy parameter, i.e. $\Delta\lambda_{i}$, are ``small enough'' across homotopy iterations, the method is able to track in expectation an $r$-optimal solution from source to target problem. As shown in Algorithm~\ref{alg:hsgd}, H-SGD takes as input an approximate solution for the problem associated with $f(w,\,0)$, i.e. $w_0$, and is then characterized by two loops: in the outer loop (homotopy iterations) the method defines a new objective function by increasing the homotopy parameter (line 5 in Algorithm~\ref{alg:hsgd}), and in the inner loop (warm-started SGD iterations) the current homotopy problem is approximately solved with $k$ iterations of SGD starting from the previously derived approximate solution, i.e. $w_{i-1}$ (line 6 in Algorithm~\ref{alg:hsgd}).
%statement on delta lambda
Different functions $h:\mathbb{N}\rightarrow (0,1]$ to determine the increment $\Delta\lambda_i$ in the homotopy parameter at each homotopy iteration can be used. As shown in our analysis, this function greatly impacts on the method's properties and convergence rate. 
In particular, our theoretical analysis confirms that, when a specifically designed scheme for $\Delta \lambda_i$ is deployed, i.e. exponentially decreasing schedule, H-SGD is effective in guaranteeing a global linear rate of convergence to a neighborhood of a minimizer of our target problem, while, given the same setting, vanilla SGD can only ensure a global sublinear rate of convergence.

%exploit stronger local structures of the problem
\begin{algorithm}
\caption{Homotopy-Stochastic Gradient Descent (H-SGD)}\label{alg:hsgd}
\begin{algorithmic}[1]
%\State \textbf{input}: $ {w}_0\in\mathbb{R}^d \quad  \text{s.t.}\quad \mathrm{E}\left[ f(w_0,\,0) \right] - f(w^*,\,0)\leq r $%\Comment{Initialize the iterate}
\State \textbf{input:} ${w}_0\in\mathbb{R}^d $, $n\in \mathbb{N},\,k\in \mathbb{N}$, $h:\mathbb{N}\rightarrow (0,1]$ with $\sum_{i=1}^{n}h(i)=1$ and $\alpha>0$ %\Comment{Initialize the number of homotopy steps}
\State \textbf{initialization:} $i=0,\,\lambda_{0} = 0$%\Comment{Initialize the homotopy parameter and its increase}
%\State $k>0\,,k\in \mathbb{Z}$%\Comment{Define the number of iterations to be performed}
\For{$i=1,\dots, n$}
\State $\Delta\lambda_{i} \gets h(i)$
\State $\lambda_{i}\gets \lambda_{i-1}+\Delta\lambda_{i}$
%\State $k>0$\Comment{Number of iterations to be performed}
\State {$w_{i}\gets$SGD}{$\left(w_{i-1}, \alpha, k, f(\cdot,\lambda_{i})\right)$}%\Comment{Approximately solve $H(\theta, \lambda_i)$}
%\State $i\gets i+1$
\EndFor
\State \textbf{output:} $w_{i}$
\end{algorithmic}
\end{algorithm}

\subsection{Related Work}
Finding a solution of Problem~\ref{eq: main_problem} when the objective function is non-convex is often quite challenging. % given the multiple saddle points and local minima that might characterize the landscape.
Different heuristics hence have been proposed to speed up and improve the optimization of such problems, many of which to be used in combination with stochastic first-order methods such as SGD.
In this regard, the proposed method, despite being new in its general formulation and analysis, finds many interesting similarities and connections with existing heuristics in the machine learning literature, e.g.~\citep{10.1145/1553374.1553380, abs-1207-0580, 10.5555/3042817.3043064}. 
We now briefly discuss some of the state-of-the-art optimization techniques and initialization strategies for solving Problem~\ref{eq: main_problem} that are most related to H-SGD, drawing connections with existing and ongoing research works and in the hope that our analysis can also lead to a new interpretation of some widely used techniques which so far lack a more rigorous theoretical description and analysis. %Then, we also list the major contributions of our work and stress the novelty of our analysis.

\noindent{}\textbf{Graduated Optimization.}
The graduated optimization approach~\citep{10.5555/30394}, also known as \textit{coarse-to-grained optimization method}, is a general heuristic to solve complex non-convex problems that relies on the basic principles of the homotopy method. As the name suggests, at first a coarse-grained and ``easy-to-solve'' version of the target problem is generated via a smoothing operation. The method then proceeds by gradually refining the problem versions, using the previous solution as initial point. Graduated optimization has been utilized explicitly and implicitly as heuristic in many machine learning and computer vision applications, e.g. object localization~\citep{seeing_blur}, manifold learning~\citep{nonlinear_dim_reduction}, optical flow~\citep{Bro11a}. Unfortunately, many of these techniques have practical and/or theoretical gaps, as they generally lack a rigorous running time and convergence analysis, and/or, as in~\citep{DBLP:conf/emmcvpr/MobahiF14} and~\citep{DBLP:conf/icml/HazanLS16}, they rely on an expensive method, i.e. Gaussian smoothing, to construct coarse-grained versions of the original target problem.
Regarding theoretical contributions on graduated optimization methods for solving Problem~\ref{eq: main_problem},~\citet{DBLP:conf/icml/HazanLS16} are the first and only, to the best of our knowledge, to provide a theoretical analysis for the running time and convergence rate of a graduated optimization method based on an approximate, yet still expensive, type of Gaussian smoothing and SGD. Unfortunately, their analysis shows two major limitations. First, it relies on their Gaussian smoothing approximation as homotopy map, which limits the generality of the conducted analysis, while our analysis is independent from the specific formulation of the homotopy map used. Second, the analysis is based on the assumption of local strong convexity, which is a quite strong requirement and hence might lead to considerably smaller local regions than those considered in our analysis~\citep{PL_lecture}. 
To conclude this short overview on graduated optimization, many successful optimization heuristics proposed in the machine learning literature are implicitly related to graduated optimization and, consequently, to homotopy methods, such as curriculum learning~\citep{10.1145/1553374.1553380}, simulated annealing~\citep{Kirkpatrick671}, noise injection techniques~\citep{abs-1207-0580}, smart initialization~\citep{10.5555/3042817.3043064} and layer-wise pretraining~\citep{10.5555/2976456.2976476}. 
%many other optimzation heuristics are realted to different degrees, also indirectly to homotopy methods, such as curriculum learning, simulated annealing, noise injection, and layerwise pre-training. 
%\noindent{}\textbf{Homotopy for Sparse Least Squares.}
%\noindent{}\textbf{Homotopy for SVM.}

\noindent{}\textbf{Transfer Learning.} %transfer learning as smart initialization-- it is related to homotopy methods. First connection has been shown by us. We proposed a method based on homotopy methods for transferring knowledge across different tasks. The method comes also together with a theoretical analysis of optimality tracking. no convergence rate though + stringent analysis. PLus side wrt to previous is that the analysis prescindes from the specific homotopy used so it is more flexible
Due to the massive amount of computational resources required by the development of modern machine learning applications, the community has started to explore the possibility of re-using learned parameters across different tasks, leading to the development of many new transfer-learning algorithms, e.g.~\citep{Torrey_transferlearning, 10.1109/TKDE.2009.191, NIPS2014_5347, Gargiani2020TransferringOA}. A simple yet often effective way to transfer knowledge across different tasks consists in using warm-start initialization. In this perspective, transfer-learning boils down to a sort of smart initialization heuristic. A first connection between homotopy methods and transfer-learning was underlined in~\citep{Gargiani2020TransferringOA}. The authors propose a transfer-learning algorithm based on the homotopy method and SGD via the definition of a homotopy map that transforms a source task into a target task. The method comes together with a general theoretical analysis that is independent from the specific homotopy map adopted and shows that, under some assumptions, the algorithm can track in expectation an approximate solution from source to target task, i.e. optimality-tracking. Unfortunately, the method's analysis is limited as it only considers constant increments of the homotopy parameter, which automatically degrades the linear rate of the local solver to a sublinear one for the homotopy-based method. In addition, as in~\citep{DBLP:conf/icml/HazanLS16}, the analysis relies on the assumption of local strong convexity, which might hold in a significantly smaller neighborhood of the minimizers than the PL condition~\citep{PL_lecture}.

\section{Theoretical Analysis}\label{sec:theoretical_analysis}
In this section, we provide a general theoretical analysis of H-SGD as described in Algorithm~\ref{alg:hsgd}. 
In particular, after discussing the required underlying assumptions (Section~\ref{sec:assumptions}), and the fundamental theoretical preliminaries (Section~\ref{sec:th_preliminaries}), first we analyze the optimality tracking properties of the proposed method (Section~\ref{sec:opt_tracking}), and then we show that, with a specifically designed scheme for the homotopy parameter, H-SGD enjoys linear convergence to a minimizer's neighborhood (Section~\ref{sec:lin_conv}). 
The analysis we conduct is independent from the specific type of homotopy map adopted and it applies to any scenario where the assumptions listed in Section~\ref{sec:assumptions} hold.

Recall that the proposed method is based on sequentially and approximately solving a series of $n$ unconstrained parametric problems of the form%$\gamma+1$ (or $\gamma$ in case we already dispose of an approximate solution for the source problem as in~\ref{}) unconstrained parametric problems of the form%TODO: solo gamma non gamma+1, come nello schema algoritmico perche' torna meglio con l analisi poi. riguardo a esperimento con logistic regression, specifica che viene calcolata w_0 e tanto essendo convesso il problema di partenza non e' un problema l inizializzazione e si converge in poche iterazioni a una buona approssimazione
\begin{equation}
\arg\min_{w\in\mathbb{R}^d} f(w,\,\lambda_i)\,,\quad \forall i=1,\dots,n\,,
\end{equation}
where $\lambda_{i}< \lambda_{i+1}$, $\lambda_{n}=1$, $\lambda_i \in (0,1]$. %Each problem is warm-started with the approximate solution of the previous one. 
In addition, H-SGD relies on the availability of an approximate solution $w_0$ for the source problem with $\lambda_0 = 0$ as starting point. 
We use $w_i$ to denote the derived approximate solution for the problem associated with parameter $\lambda_i$ that is obtained by applying $k>0$ iterations of SGD starting from the previously derived approximate solution for the problem with parameter $\lambda_{i-1}$, $\forall i=1,\dots, n$. 

Notice that $w_{i-1,t}=w_{i-1,t}(\xi_{[i-1,t-1]})$ for $t=1,\dots,k$ with $w_{i-1,k} = w_{i}(\xi_{[i]})$ is used to refer to the random vector generated at the $i$-th homotopy iteration after $t$ iterations of SGD, where $\xi_{[i-1,t-1]}=(w_0,\xi^1_0,\dots,\xi^1_{k-1},\dots,\xi^{i-1}_{0},\dots,\xi^{i-1}_{k-1},\xi^{i}_{0},\dots,\xi^{i}_{t-1})$ and $\xi_{[i]}=(w_0,\xi^1_0,\dots,\xi^1_{k-1},\dots,\xi^{i-1}_{0},\dots,\xi^{i-1}_{k-1},\xi^{i}_{0},\dots,\xi^{i}_{k-1})$ with $\xi_{[0]}=w_0$ are used to refer to the collection of all random sources up to the current iteration. We use $U^*(\lambda)$ to denote the set of local (and global) minimizers of the parametric Problem~\ref{eq: main_problem}.

\subsection{Assumptions}\label{sec:assumptions}
We now list and discuss the assumptions that we consider throughout our analysis. Together with the standard smoothness and bounded variance assumptions, we also introduce three regularity assumptions, which describe the localization of the solution map and how the objective function $f$ changes by varying the homotopy parameter across iterations. In addition to these assumptions, we also consider a more general and local version of the standard PL condition.%~\citep[see][for more details on the standard PL condition]{10.1007/978-3-319-46128-1_50}. 
 Consequently, unlike the settings considered in~\citep{10.1007/978-3-319-46128-1_50} and~\citep{pmlr-v89-vaswani19a}, where the standard PL condition is unrealistically required to hold globally, ours is often encountered in many different non-convex scenarios~\citep[see][for more details]{PL_lecture}. %Therefore, even if quite stringent, it often hold in real scenarios.
%In this section, we list and discuss the local assumptions on the problem structure that we consider for our analysis. These assumptions are only required to hold in a neighborhood of the minimizers. Therefore, even if more stringent than the ones listed in Section~\ref{sec:global_ass}, they often hold in real scenarios.

%For ease of notation, $w$ is used both to refer to the random variable and its realizations.
\begin{assumption}[existence of a regular localization of the solution map]\label{ass: regularity_0}
Assume there exists a set $\Omega\subseteq \mathbb{R}^d\times [0,1]^z$ such that $W^*(\lambda)\coloneqq\Omega \cap U^*(\lambda)$ and $\Sigma\coloneqq \left\{ (y, \lambda)\in \mathbb{R}^d \times [0,1]^z\,|\, y\in W^*(\lambda)\right\}$ are both non-empty and connected. Moreover, we assume that for a given $\lambda$ all the points in $W^*(\lambda)$ are associated with the same objective function value, which we denote as $f^*(\lambda) \coloneqq f(y, \lambda)$ for all $y\in W^*(\lambda)$.
\end{assumption}
Notice that Assumption~\ref{ass: regularity_0} does not imply vector-valued solutions of the parametric Problem~\ref{eq: main_problem}.

%\subsubsection{Global Assumptions}\label{sec:global_ass}
\begin{assumption}[regularity 1]\label{ass: regularity_1}
Assume there exists $\delta>0$ such that
\begin{equation}\label{eq:reg_1}
\vert f(w, \tilde{\lambda}) - f(w, \hat{\lambda}) \vert \leq \delta \Vert \tilde{\lambda} - \hat{\lambda} \Vert,\,\quad \forall w\in \mathbb{R}^d,\,\forall \tilde{\lambda},\,\hat{\lambda}\in [0,1]^z\,.
\end{equation}
\end{assumption}

\begin{assumption}[regularity 2]\label{ass: regularity_2}
Assume there exists $\gamma>0$ such that
\begin{equation}\label{eq:reg_2}
\vert f^*(\tilde{\lambda}) - f^*(\hat{\lambda}) \vert \leq \gamma \Vert \tilde{\lambda} - \hat{\lambda} \Vert,\,\quad\forall  \tilde{\lambda},\,\hat{\lambda}\in [0,1]^z\,.
\end{equation}
\end{assumption}

\begin{assumption}[$L$-smoothness]\label{ass: L_smoothness}
Assume there exists $L>0$ such that
\begin{equation}
\Vert \nabla_{w} f(\tilde{w}, \lambda) -  \nabla_{w} f(\hat{w}, \lambda)  \Vert \leq L\,\Vert \tilde{w} - \hat{w} \Vert,\,\quad \forall \tilde{w}, \hat{w}\in \mathbb{R}^d,\,\forall {\lambda}\in [0,1]^z\,.
\end{equation}
\end{assumption}

See Remark~\ref{app:remark} in Section~\ref{sec:additional_remarks} of the Appendix for more details on Assumptions~\ref{ass: regularity_1} and~\ref{ass: regularity_2}.
%\comments{is Assumption 4.3 a sort of ``local Lipschitz continuity'' of $f$ function in $\lambda$?  }

%\maty{add remark and explicative example to explain and illustrate the regularity assumptions}

%\subsubsection{Local Assumptions}
%level sets
%\begin{definition}[``stochastic'' sublevel sets]\label{def: sublevel_sets}
%Let $w_i^*$ be a minimizer of $f(\cdot,\,\lambda_i)$. In addition let $c \geq 0$ and $\lambda_i \in [0,1]^{z}$. Consider the following compact set of $d$-%dimensional real random vectors 
%\begin{equation}
%\mathcal{L}_{c,\,\lambda_i}\coloneqq \left\{ m\sim P_m \,\,\text{s.t.}\,\, 0\leq \mathrm{E}_m \left[ f(m, \lambda_i) \right] -f(w_i^*, \lambda_i) \leq c  %\right\}\,.
%\end{equation}
%%We define $\mathcal{L}_{c,\,\lambda_i}$ as the biggest subset of $d$-dimensional real random vectors of $\mathcal{V}_{c,\,\lambda_i}$ where Assumptions~\ref{ass: bounded_Var} and~\ref{ass: PL_condition} hold.
%%We define $\mathcal{L}_{c,\,\lambda}$ as the largest compact and connected subset of $d$-dimensional vectors of $\mathcal{V}_{c,\,\lambda}$ such that $w^*\in \mathcal{L}_{c,\,\lambda}$.%and $\mathrm{E}_w\left[ f(w,\,\lambda) \right] - f(w^*,\,\lambda)\geq 0$ for all $w\in\mathcal{L}_{c,\,\lambda}$. 
%\end{definition}
%\quoc{how to ensure that $\mathcal{L}_{B, \lambda}$ is non-empty?}
%\maty{attempt answer: by definition that set contains $w^*$, so it can not be empty. $\mathcal{L}_{B, \lambda}$ should look like $[a,b]$. Can we arrive to that by considering that first derivatives are continuous and $w^*$ is a local minimizer by definition?}
%bounded variance
\begin{assumption}[bounded ``variance'']\label{ass: bounded_Var}
Consider $f(w,\,\lambda)$ with $\lambda \in [0,1]^z$ and let $g(w,\,\xi,\,\lambda)$ be the stochastic estimate of the true gradient $\nabla_{w} f(w,\,\lambda)$ used in SGD with noise $\xi$. %computed at the $i$-th homotopy iteration after $t$ iterations of SGD.
Assume that 
%Assume that there exists $B_{\sigma}\geq \frac{\sigma^2}{2\mu}+C_{\sigma}$ with $C_{\sigma}>0$ and $\sigma^2\geq 0$ such that
\begin{equation}
\mathrm{E}_{\xi}\left[ g(w,\,\xi,\, \lambda)\right] = \nabla_{w} f(w,\,\lambda),\,\quad \forall w \in \mathbb{R}^d,\,\forall \lambda \in [0,1]^z.
\end{equation}
and that there exists $\sigma^2\geq 0$ such that
\begin{equation}
\mathrm{E}_{\xi}\left[\Vert g(w,\,\xi,\, \lambda)-\nabla_{w} f(w,\,\lambda)\Vert^2 \right]\leq \sigma^2,\,\quad \forall w \in \mathbb{R}^d,\,\forall \lambda \in [0,1]^z.
\end{equation}
\end{assumption}

%PL condition
\begin{assumption}[``expected'' PL condition]\label{ass: PL_condition}
Consider $f(w,\,\lambda_i)$ with $\lambda_i \in [0,1]^z$ and let $w_{i-1,t}=w_{i-1,t}(\xi_{[i-1,t-1]})$ denote the iterate that is obtained at the $i$-th homotopy iteration after $t$ iterations of SGD with $t\leq k$.
Assume that there exist $B> \frac{\sigma^2}{2\mu}$ and $\mu>0$ such that, if $\mathrm{E}_{\xi_{[i-1,t-1]}} \left[ f(w_{i-1,t}, \lambda_i) \right] -f^*(\lambda_i) \leq B $, then
\begin{equation}
\mathrm{E}_{\xi_{[i-1,t-1]}}\left[ \Vert \nabla_w f(w_{i-1,t},\,\lambda_i) \Vert^2 \right] \geq 2\mu\cdot \left[ \mathrm{E}_{\xi_{[i-1,t-1]}}\left[  f(w_{i-1,t}, \lambda_i) \right] -f^*(\lambda_i)\right]\,.
\end{equation}
\end{assumption}

See Remark~\ref{remark_PL} for additional details on Assumption~\ref{ass: PL_condition}.

\subsection{Fundamental Theoretical Preliminaries}\label{sec:th_preliminaries}
Before proceeding with the main theoretical contributions, we revise and adjust the existing results in the literature on global error bounds of SGD, i.e.~\citep{pmlr-v89-vaswani19a}, to also hold in the considered setting. The extended results are then used for the derivations in Section~\ref{sec:opt_tracking} and~\ref{sec:lin_conv}.  
%invariant level sets
\begin{proposition}\label{prop_1}
Consider $f(w,\,\lambda_i)$ with $\lambda_i \in [0,1]^z$ and let $w_{i-1,t}=w_{i-1,t}(\xi_{[i-1,t-1]})$ denote the iterate obtained at the $i$-th homotopy iteration by applying $t$ iterations of SGD with $t\leq k-1$ and $\alpha\leq \frac{1}{L}$. Under Assumptions~\ref{ass: regularity_0} and~\ref{ass: L_smoothness}-\ref{ass: PL_condition}, if $\mathrm{E}_{\xi_{[i-1,t-1]}} \left[ f(w_{i-1,t}, \lambda_i) \right] -f^*(\lambda_i) \leq B$, then $\mathrm{E}_{\xi_{[i-1,t]}} \left[ f(w_{i-1,t+1}, \lambda_i) \right] -f^*(\lambda_i) \leq B$.
\end{proposition}
\begin{proof}
See Section~\ref{sec:proof_proposition} in the Appendix for a proof.
\end{proof}

\begin{theorem}\label{theorem_SGD}
Consider the minimization of $f(w,\,\lambda_i)$ with $\lambda_i \in [0,1]^z$ via SGD. Let $w_{i-1}=w_{i-1}(\xi_{[i-1]})$ be the random initial point associated with the $i$-th homotopy iteration with $\mathbb{E}_{\xi_{[i-1]}} \left[f(w_{i-1}, \lambda_i) \right] -f^*(\lambda_i) \leq B$ and $w_{i-1,t}=w_{i-1,t}(\xi_{[i-1,t-1]})$ denote the $t$-th SGD iterate with $t\leq k$. Under Assumptions~\ref{ass: regularity_0} and~\ref{ass: L_smoothness}-\ref{ass: PL_condition}, SGD with a constant step-size $\alpha\leq \frac{1}{L}$ attains the following convergence rate to a minimizer's neighborhood
\begin{equation}\label{eq:convergence_SGD}
    \mathrm{E}_{\xi_{[i-1,t-1]}}\left[f(w_{i-1,t},\,\lambda_i)-f^*(\lambda_i) \right]\leq \rho^t \mathrm{E}_{\xi_{[i-1]}}\left[f(w_{i-1},\,\lambda_i)-f^*(\lambda_i) \right] + \frac{\sigma^2}{2\mu}\,, 
\end{equation}
with $\rho\coloneqq\left(1-\alpha\mu \right)$. 
With $\alpha=\frac{1}{L}$, we obtain $\rho=\left(1-\frac{\mu}{L} \right)$.
\end{theorem}
\begin{proof}
See Section~\ref{sec:proof_th_sgd} in the Appendix for a proof.
\end{proof}

\subsection{Optimality Tracking}\label{sec:opt_tracking}
%intro
In the following, we define the function $\phi_v(\lambda_i) \coloneqq \mathrm{E}_v\left[ f(v,\,\lambda_i) \right]-f^*(\lambda_i)$ where $v$ is $d$-dimensional real random vector.
As in~\citep{Gargiani2020TransferringOA} but in a more relaxed setting, i.e. local PL in place of local strong-convexity, we study the optimality tracking properties of H-SGD. In particular, under the considered assumptions and by exploiting the previously introduced results on the convergence of SGD, with Theorem~\ref{theorem_1} we characterize the maximum allowed variation of the homotopy parameter across homotopy iterations of H-SGD such that, if $\phi_{w_i}(\lambda_i)\leq r$, then also $\phi_{w_{i+1}}(\lambda_{i+1})\leq r$. The upper bound that we derive depends on the number of iterations $k$ performed with SGD as well as on the convergence characteristics of SGD and the structural properties of the parametric problems. This result applied recursively across homotopy iterations leads to conclude that, if we adopt a ``small enough'' increasing step for the homotopy parameter, H-SGD can track in expectation an $r$-optimal solution from source to target problem.   

Before proceeding with the actual optimality tracking analysis (Theorem~\ref{theorem_1}), we study the conditions on $w_i$ and $\Delta \lambda_{i+1}$ such that $\phi_{w_i}(\lambda_{i+1})\leq B$, where $w_i$ is  the approximate solution of the problem associated with parameter $\lambda_i$ that is also used as starting point for the next parametric problem. 
\begin{lemma}\label{lemma:initial_conditions}
Assume $\Vert \lambda_{i+1} - \lambda_i \Vert \leq \epsilon$, $0\leq \epsilon < \frac{B}{\delta + \gamma}$ and let $w_i$ denote the $i$-th iterate of Algorithm~\ref{alg:hsgd} with $\alpha\leq \frac{1}{L}$. Under Assumptions~\ref{ass: regularity_0}-~\ref{ass: regularity_2} and~\ref{ass: L_smoothness}-~\ref{ass: PL_condition}, if $\phi_{w_i}(\lambda_i)\leq B-(\delta+\gamma)\epsilon$, then $\phi_{w_i}(\lambda_{i+1})\leq B$.
In addition, let $k_{\max} \coloneqq \Bigg\lceil\log_{\rho}\left( 1 - \frac{2\mu(\delta + \gamma)\epsilon + \sigma^2}{2\mu B} \right)\Bigg\rceil$. 
%\begin{equation}
%k_{\max} \coloneqq \Bigg\lceil\log_{\rho}\left( 1 - \frac{2\mu(\delta + \gamma)\epsilon + \sigma^2}{2\mu B} \right)\Bigg\rceil\,.
%\end{equation}  
If $\phi_{w_i}(\lambda_{i+1})\leq B$, $0\leq \epsilon< \frac{1}{\delta + \gamma}\left(B - \frac{\sigma^2}{2\mu} \right)$ and $k\geq k_{\max}$, then $\phi_{w_{i+1}}(\lambda_{i+1})\leq B-(\delta+\gamma)\epsilon$.
\end{lemma}
\begin{proof}
See Section~\ref{sec:proof_lemma} in the Appendix for a proof.
\end{proof}

\begin{theorem}\label{theorem_1}
Assume there exists $\frac{\sigma^2}{2\mu}< r\leq B$ and $\tilde{\epsilon} \coloneqq\min \left\{ \epsilon_1,\,\epsilon_2\right\}$ with
\begin{equation}
\epsilon_1\coloneqq \frac{1}{(\delta+\gamma)}(B-r),\quad\epsilon_2 \coloneqq \frac{(1-\rho^k)\, r-\sigma^2/2\mu}{\rho^k\,(\delta+\gamma)}\,.
\end{equation}
In addition, let $k_{\max}\coloneqq \Bigg\lceil \log_{\rho} \left( 1-\frac{\sigma^2}{2\mu r} \right)  \Bigg\rceil$.
%\begin{equation}
%k_{\max}\coloneqq \Bigg\lceil \log_{\rho} \left( 1-\frac{\sigma^2}{2\mu r} \right)  \Bigg\rceil\,.
%\end{equation}
Consider Algorithm~\ref{alg:hsgd} with $\alpha \leq \frac{1}{L}$, $k\geq k_{\max}$ and $\Vert \lambda_{i+1} - \lambda_i \Vert \leq \epsilon$, where $0\leq\epsilon\leq \tilde{\epsilon}$.
Under Assumptions~\ref{ass: regularity_0}-~\ref{ass: regularity_2} and~\ref{ass: L_smoothness}-~\ref{ass: PL_condition}, if $\phi_{w_i}(\lambda_i)\leq r$, then $\phi_{w_{i+1}}(\lambda_{i+1})\leq r$.
\end{theorem}

\begin{proof}
See Section~\ref{sec:proof_th_ot} in the Appendix for a proof.
\end{proof}
See Figure~\ref{fig:opt_track} in Section~\ref{app:additional_figures} of the Appendix for a graphical representation of the derived results.
\subsection{Linear Convergence Rate}\label{sec:lin_conv}
We now study the convergence rate of H-SGD and, in particular, if it is possible to recover a global linear rate of convergence to a minimizer's neighborhood. The results derived in Theorem~\ref{theorem_lin_conv} confirm that, in the considered setting and with a specifically designed schedule for the homotopy parameter, H-SGD achieves the desired rate of convergence.
\begin{theorem}\label{theorem_lin_conv}
Let $\tilde{\rho}\in\left( 1-\frac{\sigma^2}{2\mu}\frac{1}{B},1 \right)$ and consider Algorithm~\ref{alg:hsgd} with $\alpha\leq \frac{1}{L}$, $\phi_{w_0}(\lambda_0)\leq r$ with $\frac{\sigma^2}{2\mu}\frac{1}{(1-\tilde{\rho})}\leq r\leq B $ and  $k\geq \log_{\rho}(\tilde{\rho})$. In addition, let $\epsilon_1\coloneqq \frac{1}{(\delta+\gamma)}(B-r)$ and 

\begin{equation}
C_{\tilde{\rho}}\coloneqq
\begin{cases}
      1 & \text{if }k\geq \log_{\rho}(\tilde{\rho})-\log_{\rho}\left( 1 + \frac{\delta+\gamma}{\varepsilon_0} \right)\\
      \frac{\tilde{\rho}-\rho^k}{\rho^k}\frac{\varepsilon_0}{(\delta+\gamma)} & \text{otherwise,}
    \end{cases} 
\end{equation}
with $\varepsilon_0\coloneqq \mathrm{E}_{\xi_{[0]}}\left[ f(w_0,\,\lambda_0) \right] - f^*(\lambda_0)$.
Under Assumptions~\ref{ass: regularity_0}-~\ref{ass: regularity_2} and~\ref{ass: L_smoothness}-~\ref{ass: PL_condition}, if $\Vert \lambda_{i+1}-\lambda_i \Vert\leq \min\left\{e^{-\eta\,i}, \epsilon_1\right\}$ with $\eta\geq \ln{\left(C_{\tilde{\rho}}\,\tilde{\rho}\right)}$, then
\begin{equation}
\mathrm{E}_{\xi_{[i+1]}}\left[ f(w_{i+1},\,\lambda_{i+1}) \right] - f^*(\lambda_{i+1}) \leq \tilde{\rho}^{i+1}\left[ \mathrm{E}_{\xi_{[0]}}\left[  f(w_0, \,\lambda_0) \right] - f^*(\lambda_0) \right]  + \frac{\sigma^2}{2\mu}\sum_{j=0}^i \tilde{\rho}^j\,.
\end{equation}
\end{theorem}

\begin{proof}
See Section~\ref{sec:proof_th_lc} in the Appendix for a proof.
\end{proof}
\section{Experimental Evaluation}\label{sec:experiments}
In this section, we empirically validate the theoretical results derived in Theorem~\ref{theorem_lin_conv}. First, we consider a $1$-dimensional toy regression problem to illustrate and visualize some of the basic properties of H-SGD. In this easy scenario, the introduced assumptions can be trivially verified by inspection (see the Figures~\ref{fig:toy_loss}-~\ref{fig:toy_mu} in the Appendix). We then move to more complex and high-dimensional scenarios where the assumptions can not be verified. Inspired by~\citet{10.5555/3305381.3305498} and~\citet{Gargiani2020TransferringOA}, we consider the task of regressing with a neural network from input to output of a sinusoidal wave. Finally, we also consider a non-convex classification task based on the combination of logistic regression with a non-linear model for moon-shaped binary data~\citep{3603}. 
In all the considered scenarios, we compare the numerical performance of H-SGD with those of SGD and tune the step-size based on the performance of the latter. 

\subsection{Toy-Problem}\label{sec:toy_case}
We start with an easy regression problem motivated by~\citet{2016arXiv160104114M}: a $1$-dimensional neural network with $\mathrm{erf}$ as activation function (see Figure~\ref{fig:toy_net} in Section~\ref{app:additional_figures} of the Appendix for a graphical representation). We generate a synthetic dataset of $N=100$ samples, where $x_j \in \left[-1, 1 \right]$, $y_j = 3\cdot x_j + \epsilon_j$ and $\epsilon_j\sim \mathcal{N}\left(0,\,1\right)$. %See Figure~\ref{} in Section~\ref{} of the Appendix for a graphical representation.
Regarding the choice of a value for the step-size, we use an estimate $\tilde{L}$ of the smoothness constant $L$ and set $\alpha = {1}/{\tilde{L}}$.    
As loss, we use the mean squared error, which, composed with the regressing model, leads to the following non-convex optimization problem 
\begin{equation}\label{eq:obj_fun_toy_case}
w^*\,=\,\arg\min_{w\in \mathbb{R}}\,\frac{1}{N} \sum_{j=1}^N \,(y_j-\mathrm{erf}(w\cdot x_j))^2\,.
\end{equation}
By plotting the objective function with respect to $w$ (see Fig.~\ref{fig:toy_loss} in Section~\ref{app:additional_figures} of the Appendix), it is easy to observe that the PL condition holds globally, with the value of $\mu$ increasing by approaching the minimizer and $\mu\rightarrow 0$ for $w\rightarrow \pm\infty$ (see Fig.~\ref{fig:toy_mu} in the Appendix). Consequently, SGD enjoys a global linear convergence rate as proved in~\citep{pmlr-v89-vaswani19a} but the rate itself, which depends on the value of $\mu$, will dramatically worsen the further the iterates are from the minimizer, leading to a great overall sensitivity of the method in terms of convergence rate to the initialization. On the other side, H-SGD, thanks to the homotopy principle, goes around that issue by preserving the vicinity to the minimizer of the current homotopy problem at each homotopy iteration (see Figures~\ref{fig:toy_hsgd_iterates} and~\ref{fig:toy_sgd_iterates} in Section~\ref{app:additional_figures} of the Appendix). In order to achieve that, given $w_0$ as initial value, we set $y_{j,0}=w_0\cdot x_{j}$ for all $j=1,\dots, N$ and define the following homotopy transformation
\begin{equation}\label{eq:homotopy_mapping}
y_{j,\,\lambda} = \lambda \, y_j + (1-\lambda)\, y_{j,\, 0}\,.
\end{equation}
As suggested by the theory (see Theorem~\ref{theorem_lin_conv}), we select an exponentially decreasing scheme for the increment $\Delta\lambda_i$ in order to achieve linear convergence. As shown in Figure~\ref{fig:toy_gap} in the Appendix, both H-SGD and SGD enjoy a linear rate of convergence, but H-SGD shows a superior numerical performance. This is due to the fact that the method is designed such that its iterates always lie in the neighborhood of the minimizers where more favorable values of $\mu$ lead to a faster convergence. This fact allows H-SGD to enjoy a faster global convergence rate than that of SGD.  
\subsection{Regression with Deep Neural Networks}\label{sec:sine_wave}
Our second experiment is inspired by~\citet{10.5555/3305381.3305498} and~\citet{Gargiani2020TransferringOA} and focuses on studying the numerical performance of H-SGD on the task of regressing from input to output of a sinusoidal function corrupted by Gaussian noise. In particular, the input data are sampled uniformly from the interval $[-1,\, 1]$ and $y_j= \sin(10\cdot x_j) + \epsilon_j$ with $\epsilon_j\sim \mathcal{N}(0, 0.1)$ for all $j=1,\dots, 500$.
The regressor is a feedforward neural network with two hidden layers, each of $10$ units, and hyperbolic tangent as activation function. As for the previous benchmark, we use the mean squared error as loss. We employ the same values of step-size $\alpha$ and mini-batch $M$ for both H-SGD and SGD, where the step-size value is tuned based on the numerical performance of SGD for the selected mini-batch size ($M=5$).  
Regarding H-SGD, we set $y_{j,0} = x_j^2 + \epsilon_j$ with $\epsilon_j\sim\mathcal{N}(0,\,0.01)$ and employ the same homotopy mapping as in Equation~\ref{eq:homotopy_mapping}. As shown in Figure~\ref{fig:sin_gap} in the Appendix, also in this scenario H-SGD shows a superior numerical performance than SGD, i.e. H-SGD reaches a loss of $10^{-1}$ roughly $4$ times faster and achieves convergence more than $2$ times faster than SGD. 
%add one statement with why it is better and pointing to Figure 10
The superior numerical performance of H-SGD can be attributed to its ability of tracking a solution across homotopy iterations (see Figure~\ref{fig:hsgd_pred_lmbd} in the Appendix) which ensures the method to always work in the vicinity of a minimizer.
\subsection{Non-Linear Binary Classification with Logistic Regression} \label{sec:binary_classification}
Finally, we test H-SGD also on a classification benchmark with a logistic regression task. In particular, we use a 2-dimensional binary moon-shaped dataset~\citep{3603} with $1000$ samples corrupted by Gaussian noise. As the dataset is clearly not linearly separable, we opt for a cubic model, which, used in combination with the logistic regression framework, leads to a non-convex objective function where the optimization variable $w$ is the collection of the model's coefficients. For both H-SGD and SGD we use a mini-batch size of $20$ and tune the value of the step-size on the SGD's performance for that mini-batch size. Regarding H-SGD, we use as source task the one obtained considering a linear instead of a cubic model, which results in a convex and hence ``easy'' optimization problem. We then gradually increase the non-linearity of the model, i.e. non-convexity of the problem, until reaching in the final homotopy iteration the target problem with the desired cubic model. This homotopy map is obtained by multiplying the coefficients of the non-linear terms in the model by $\lambda$ as follows
\begin{equation}
\lambda (c_1\,x_{j,1}^3 + c_2\,x_{j,2}^3 + c_3\, x_{j,1}^2 + c_4\,x_{j,2}^2 + c_5\,x_{j,1}^2\,x_{j,2} + c_6\,x_{j,1}\,x_{j,2}^2) + c_7\,x_{j,1} + c_8\,x_{j,2} + c_9\,.
\end{equation}
For the homotopy parameter we adopt the increasing schedule that is suggested in Theorem~\ref{theorem_lin_conv}. H-SGD outperforms SGD by reaching an error of $0.1$ more than two times faster than SGD (see Figure~\ref{fig:hsgd_sgd_classification} in the Appendix).
   
\section{Conclusions and Future Work}
In this paper we propose a new first-order stochastic method for non-convex large-scale problems, called Homotopy-SGD (H-SGD), based on the combination of homotopy methods and SGD. This new homotopy-based optimization method allows one to exploit easy-to-solve or already-solved problems to solve new and complex ones. This is achieved by approximately and sequentially solving a sequence of optimization problems where the source problem is gradually morphed via a homotopy map into the target one. We conduct a theoretical analysis of the optimality tracking properties and convergence rate of H-SGD under some realistic and mild assumptions. 
The theoretical results are confirmed by some empirical evaluations, which also show the great potential in terms of performance of combining SGD with an homotopy strategy. In addition, H-SGD shows interesting connections with many practical existing heuristics proposed in the machine learning literature to speed up the convergence of first-order methods, allowing for a new and more rigorous interpretation of the latter.
The current major limitation of the method relies in the design of the homotopy map. Future work should focus on exploiting the specific problem structure to design optimal homotopy maps. Moreover, under additional assumptions, more theoretical results concerning the quality of the tracked solutions  could be derived.
%Then we empirically evaluate the numerical performance of our method on some benchmarks, which not only confirm the theoretical results but also show that H-SGD is competitive against SGD. 

%\subsubsection*{Author Contributions}
%If you'd like to, you may include  a section for author contributions as is done
%in many journals. This is optional and at the discretion of the authors.
%\newpage
\subsubsection*{Acknowledgments}
This work has partly been supported by the European Research Council (ERC) under the European Union’s Horizon 2020 research and innovation programme under grant no.\ 716721 as well as by the German Federal Ministry for Economic Affairs and Energy (BMWi) via DyConPV (0324166B), and by DFG via Research Unit FOR 2401. 
In addition, Q. Tran-Dinh has partly been supported by the National Science Foundation (NSF), grant. no. 1619884.

\bibliography{iclr2021_conference}
\bibliographystyle{iclr2021_conference}

\newpage
\appendix
\section{Appendix}
\section{Additional Figures}\label{app:additional_figures}

%explicative figure
%\begin{figure}[!htbp]
%  \centering
%  \hspace{-0.5cm}%\hspace{-1.5cm}
%  \begin{minipage}[t]{0.45\textwidth}
%    \includegraphics[width=1.15\textwidth]{level_sets.pdf}
%    \caption{Graphical representation of Assumption~\ref{ass: PL_condition} for the deterministic scenario. In red the region of gradient descent iterates such that $f(w_k)\leq c$ and the function is locally invex.}\label{fig:level_sets}
%  \end{minipage}
%  \hspace{.8cm}{
%  \begin{minipage}[t]{0.45\textwidth}
%  \centering
%    \includegraphics[width=1.15\textwidth]{f_E_f.pdf}
%    \caption{Graphical representation of a $1$-dimensional function with $w_t\sim\mathcal{N}\left(0,\, 1 \right)$ that verifies the ``expected'' PL condition (Assumption~\ref{ass: PL_condition}) but not the classical PL condition in the considered region.}\label{fig:PL_condition}
%  \end{minipage}}
%\end{figure}
%\FloatBarrier
%
%optimality tracking figure
\begin{figure}[!htbp]
  \centering
  %\hspace{-0.5cm}%\hspace{-1.5cm}
    \begin{minipage}[t]{0.45\textwidth}
  \centering
    \includegraphics[width=1.15\textwidth]{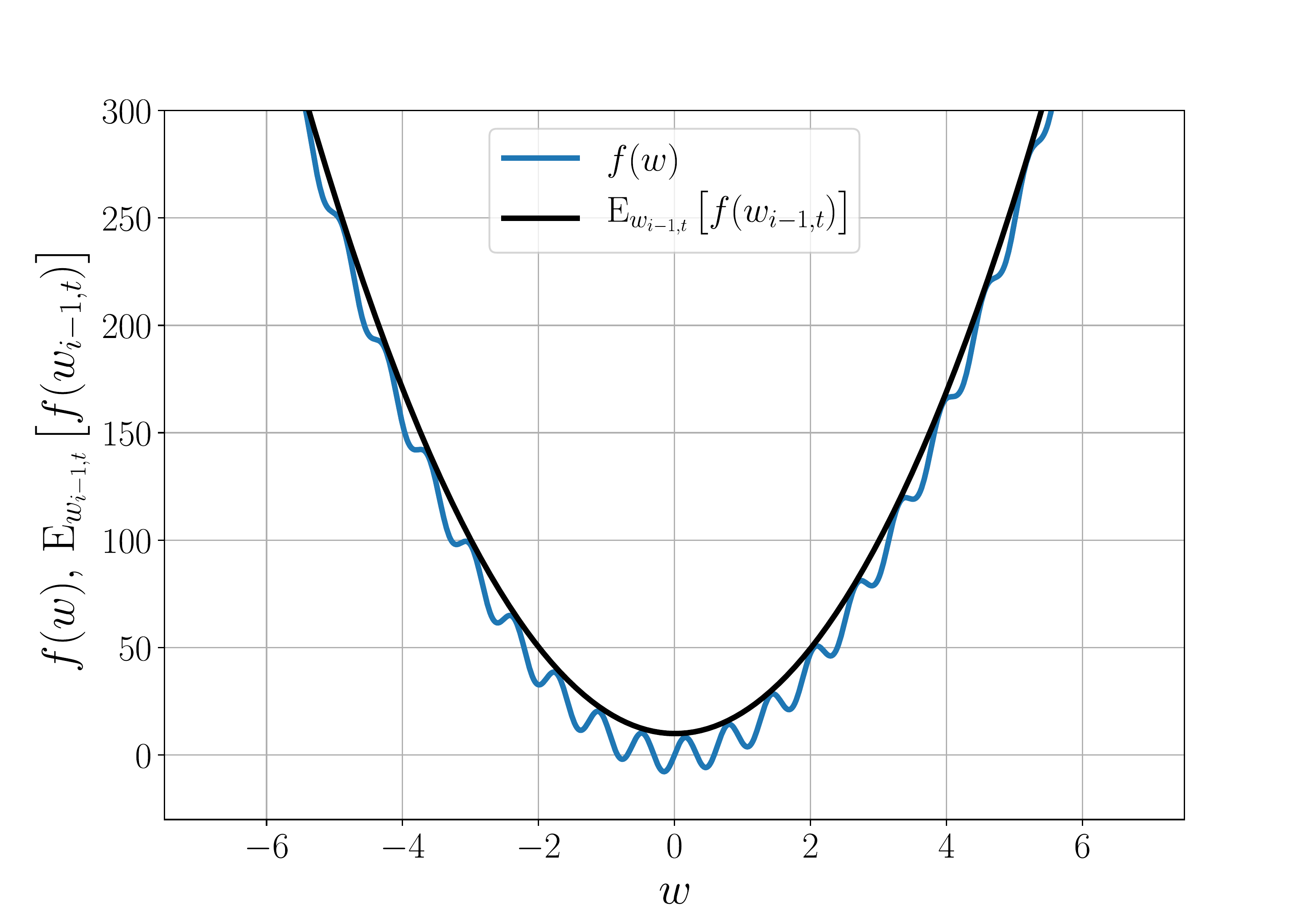}
    \caption{Graphical representation of a $1$-dimensional function with $w_{i-1,t}\sim\mathcal{N}\left(0,\, 1 \right)$ that satisfies the ``expected'' PL condition (Assumption~\ref{ass: PL_condition}) but not the classical PL condition in the considered region.}\label{fig:PL_condition}
  \end{minipage}
  \hspace{.8cm}{
  \begin{minipage}[t]{0.45\textwidth}
  \centering
    \includegraphics[width=1.15\linewidth]{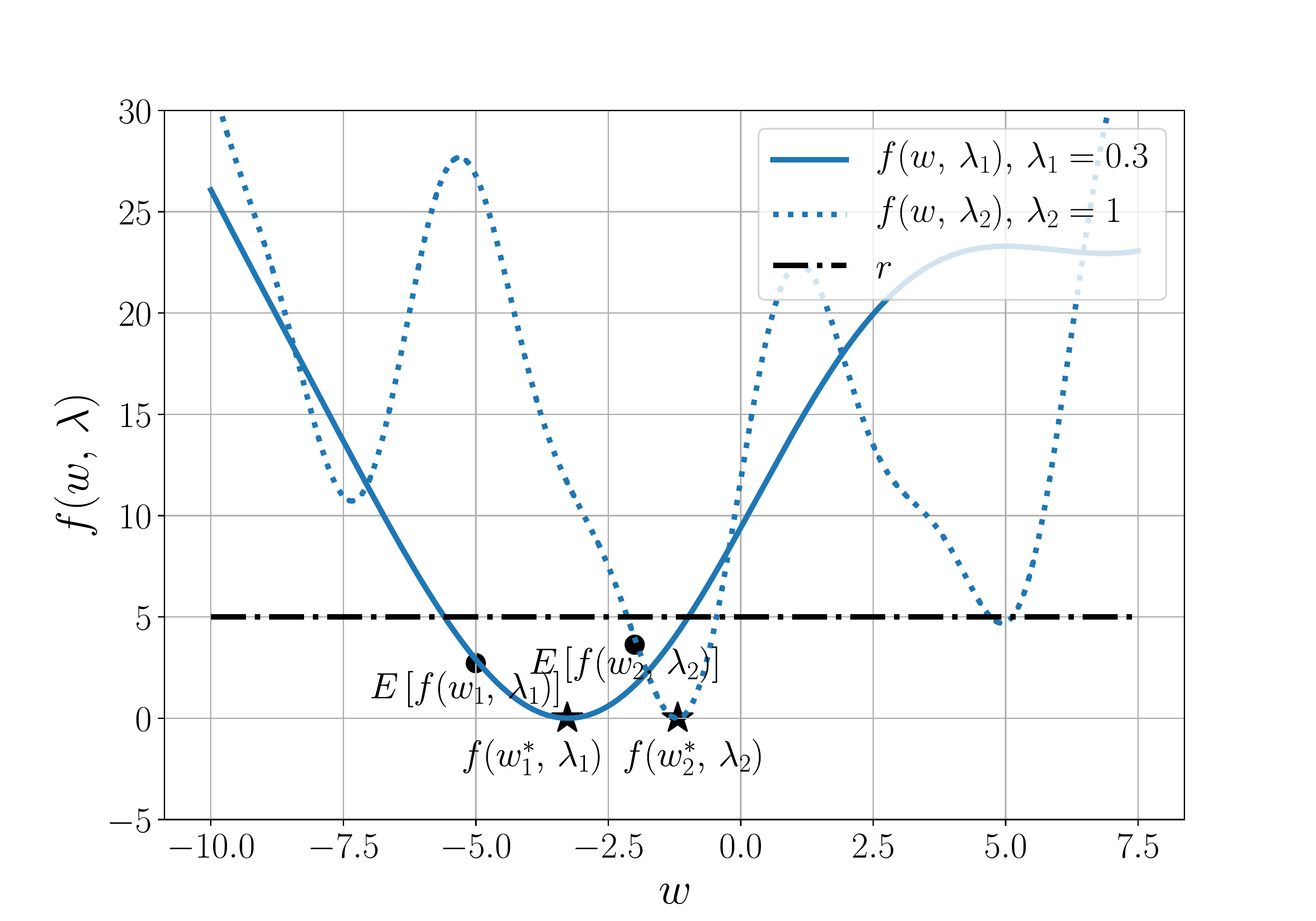} 
    \caption{Graphical representation of the results derived in Theorem~\ref{theorem_1} on the optimality tracking properties of H-SGD for a general non-convex 1-dimensional function. In particular, as shown in the figure, under the considered assumptions and for ``small enough'' variations of the homotpy parameter, H-SGD tracks in expectation an $r$-optimal solution across homotopy iterations.}\label{fig:opt_track}
  \end{minipage}}
\end{figure}
%\begin{figure}[!htbp]  
 %   \centering
 %   \includegraphics[width=\linewidth]{opt_track.pdf} 
 % 	\caption{Graphical representation of the results derived in Theorem~\ref{theorem_1} on the optimality tracking properties of H-SGD for a general non-convex 1-dimensional function. In particular, as shown in the figure, under the considered assumptions and for ``small enough'' variations of the homotpy parameter, H-SGD tracks in expectation an $r$-optimal solution across homotopy iterations.}\label{fig:opt_track}
%\end{figure}
\FloatBarrier

\begin{figure}[!htbp]
  \centering
  %\hspace{-0.5cm}
  \begin{minipage}[t]{0.45\textwidth}
    \includegraphics[width=1.15\textwidth]{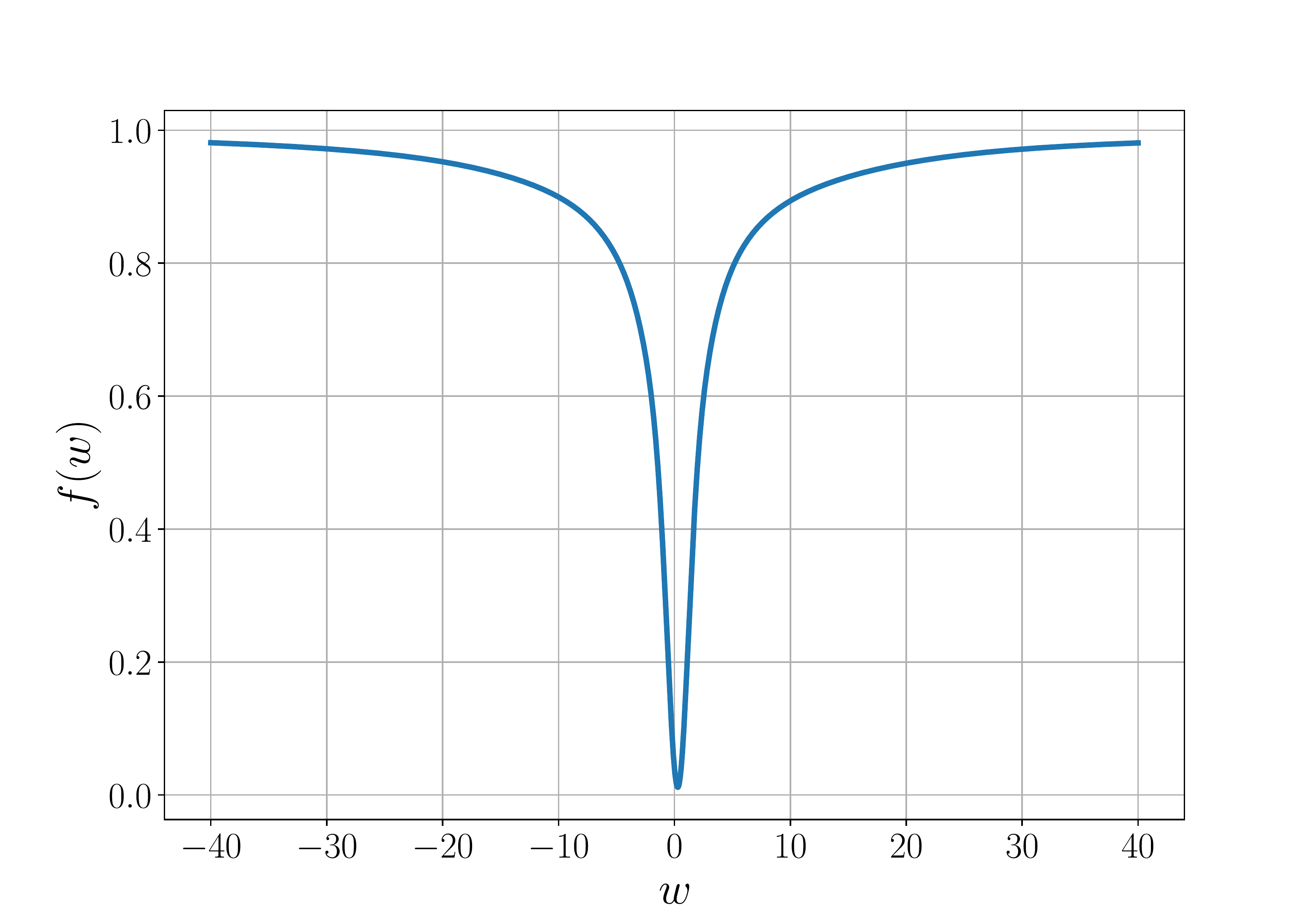}
    \caption{Graphical representation of the objective function in Problem~\eqref{eq:obj_fun_toy_case} vs $w$.}\label{fig:toy_loss}
  \end{minipage}
  \hspace{.8cm}
  \begin{minipage}[t]{0.45\textwidth}
    \includegraphics[width=1.15\textwidth]{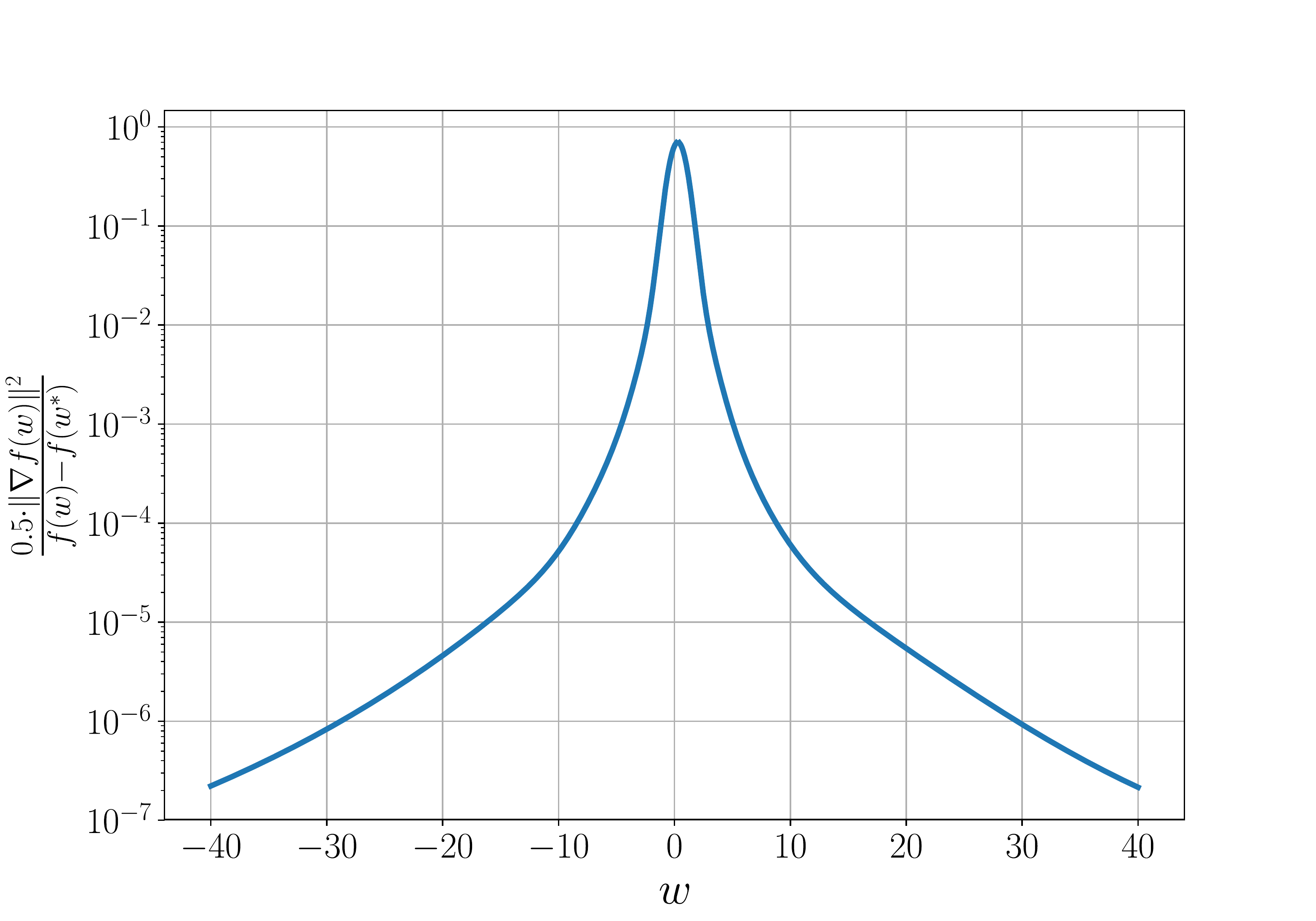}
    \caption{Visualization of the estimated $\mu$ parameter for the objective function in Problem~\eqref{eq:obj_fun_toy_case} vs $w$.}\label{fig:toy_mu}
  \end{minipage}
\end{figure}
\FloatBarrier

\begin{figure}[!htbp]
  \centering
  %\hspace{-0.5cm}
  \begin{minipage}[t]{0.45\textwidth}
    \includegraphics[width=1.15\textwidth]{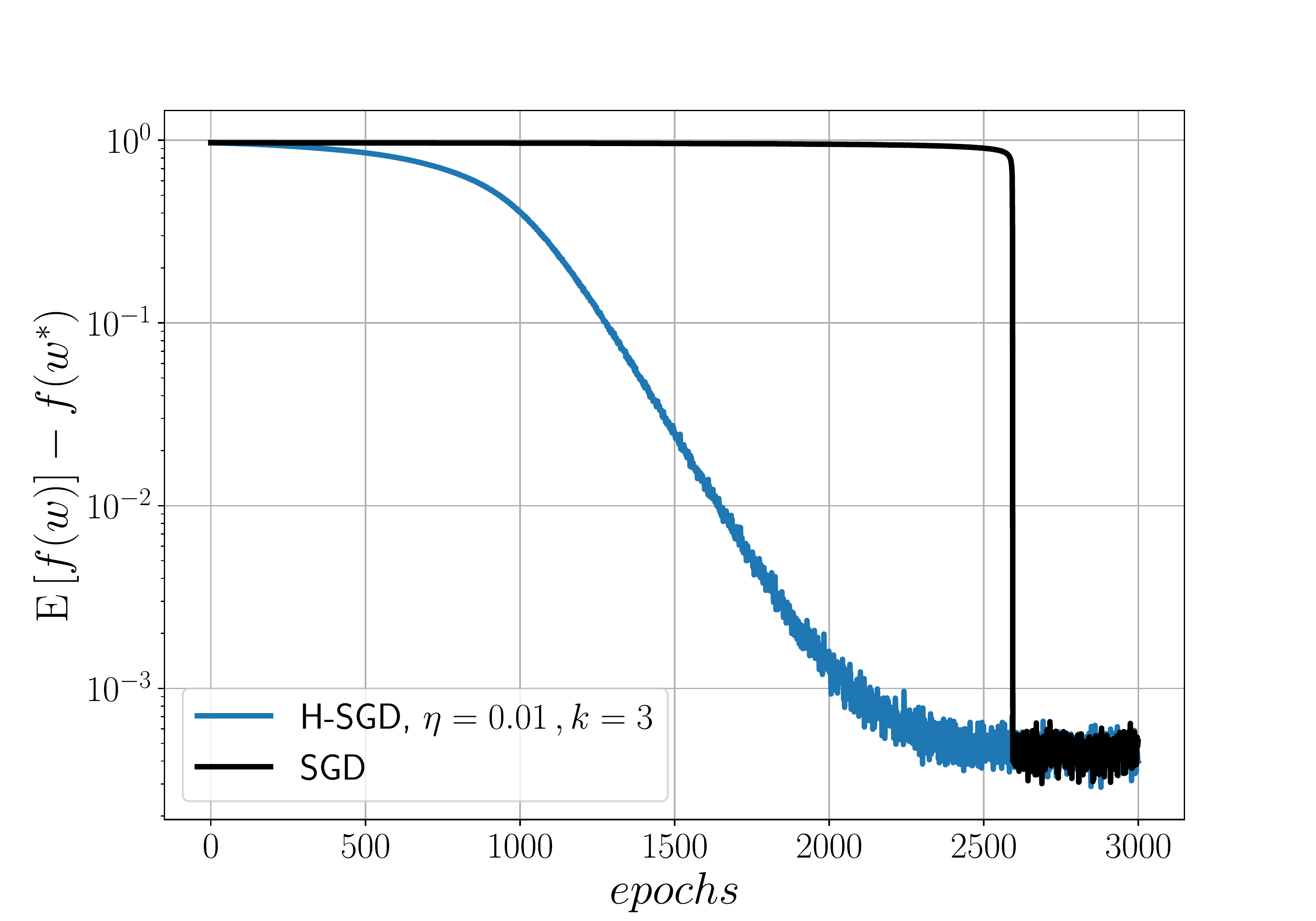}
    \caption{Expected optimality gap of H-SGD (blue) and SGD (black) averaged across 100 runs vs epochs for the toy-case described in Section~\ref{sec:toy_case}}\label{fig:toy_gap}
  \end{minipage}
  \hspace{.8cm}
  \begin{minipage}[t]{0.45\textwidth}
    \includegraphics[width=1.15\textwidth]{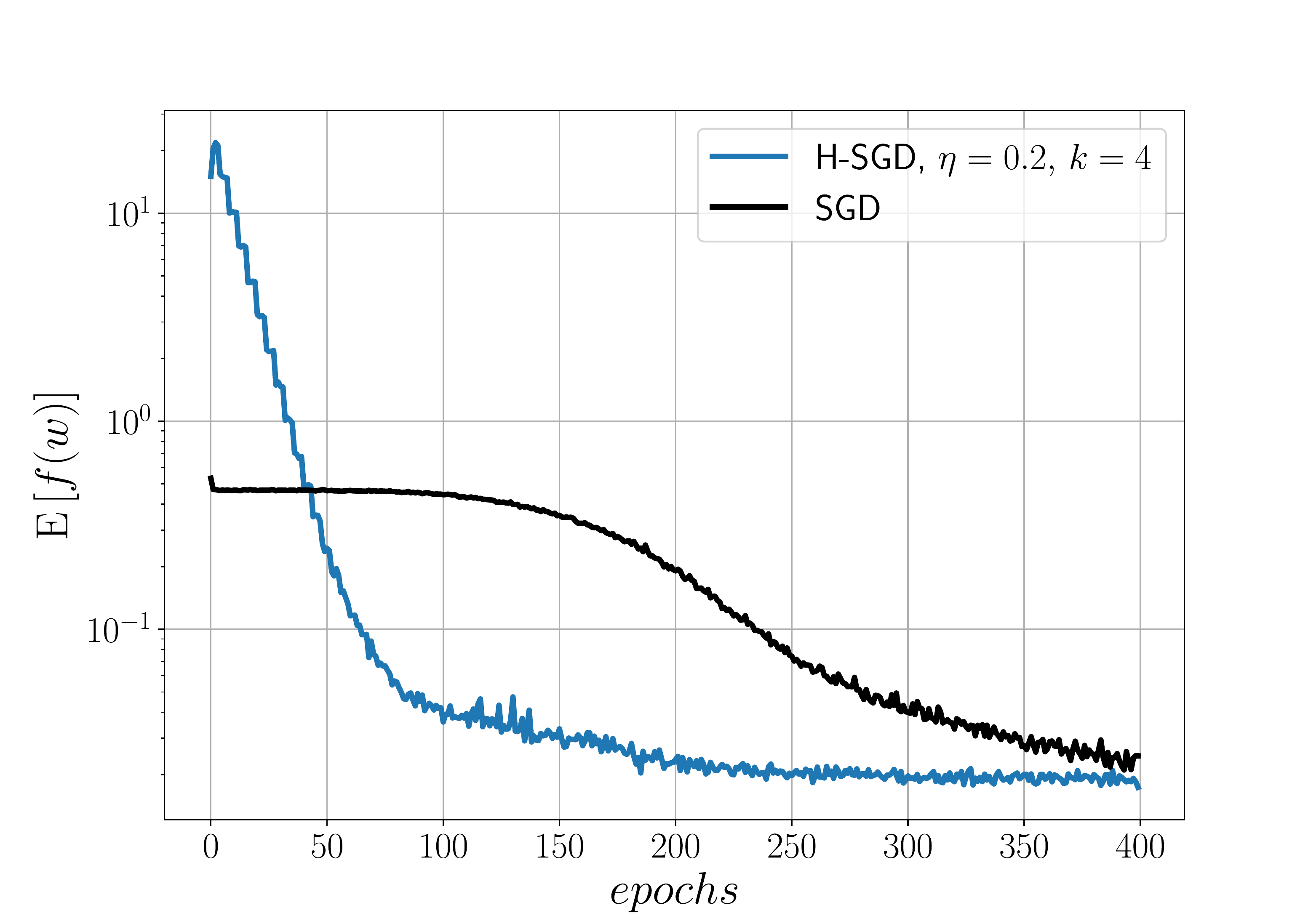}
    \caption{Expected loss of H-SGD (blue) and SGD (black) averaged across 100 runs vs epochs for the sine-wave regression case described in Section~\ref{sec:sine_wave}}\label{fig:sin_gap}
  \end{minipage}
  %\caption{}
\end{figure}
\begin{figure}[!htbp]
  \centering
  %\hspace{-0.5cm}
  \begin{minipage}[t]{0.45\textwidth}
    \includegraphics[width=1.15\textwidth]{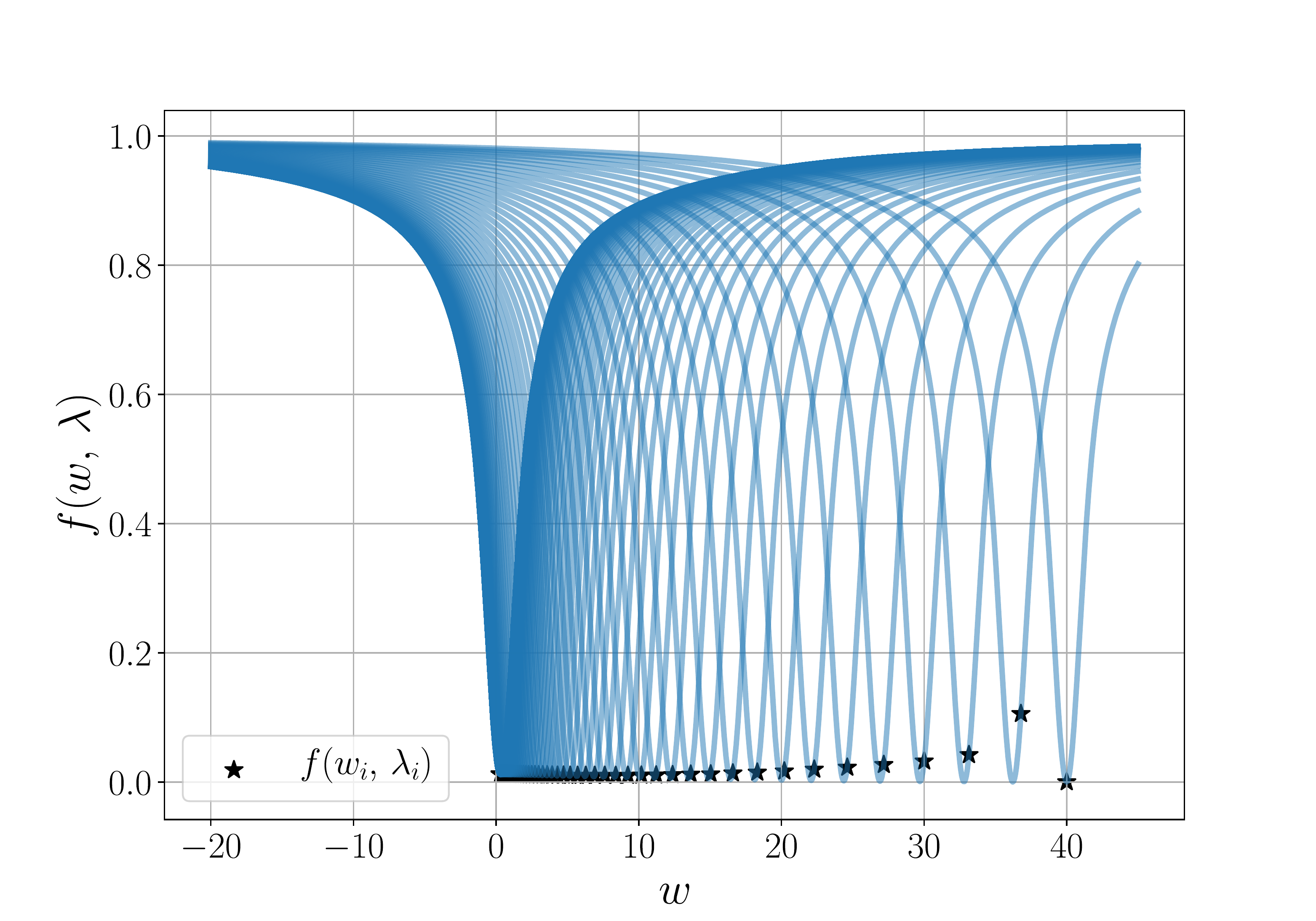}
    \caption{Visualization of the homotopy objective functions vs $w$ for different values of homotopy parameter. The black stars represent $(w_i,\,f(w_i,\,\lambda_i))$, i.e. the H-SGD iterates with associated objective function value.}\label{fig:toy_hsgd_iterates}
  \end{minipage}
  \hspace{.8cm}
  \begin{minipage}[t]{0.45\textwidth}
    \includegraphics[width=1.15\textwidth]{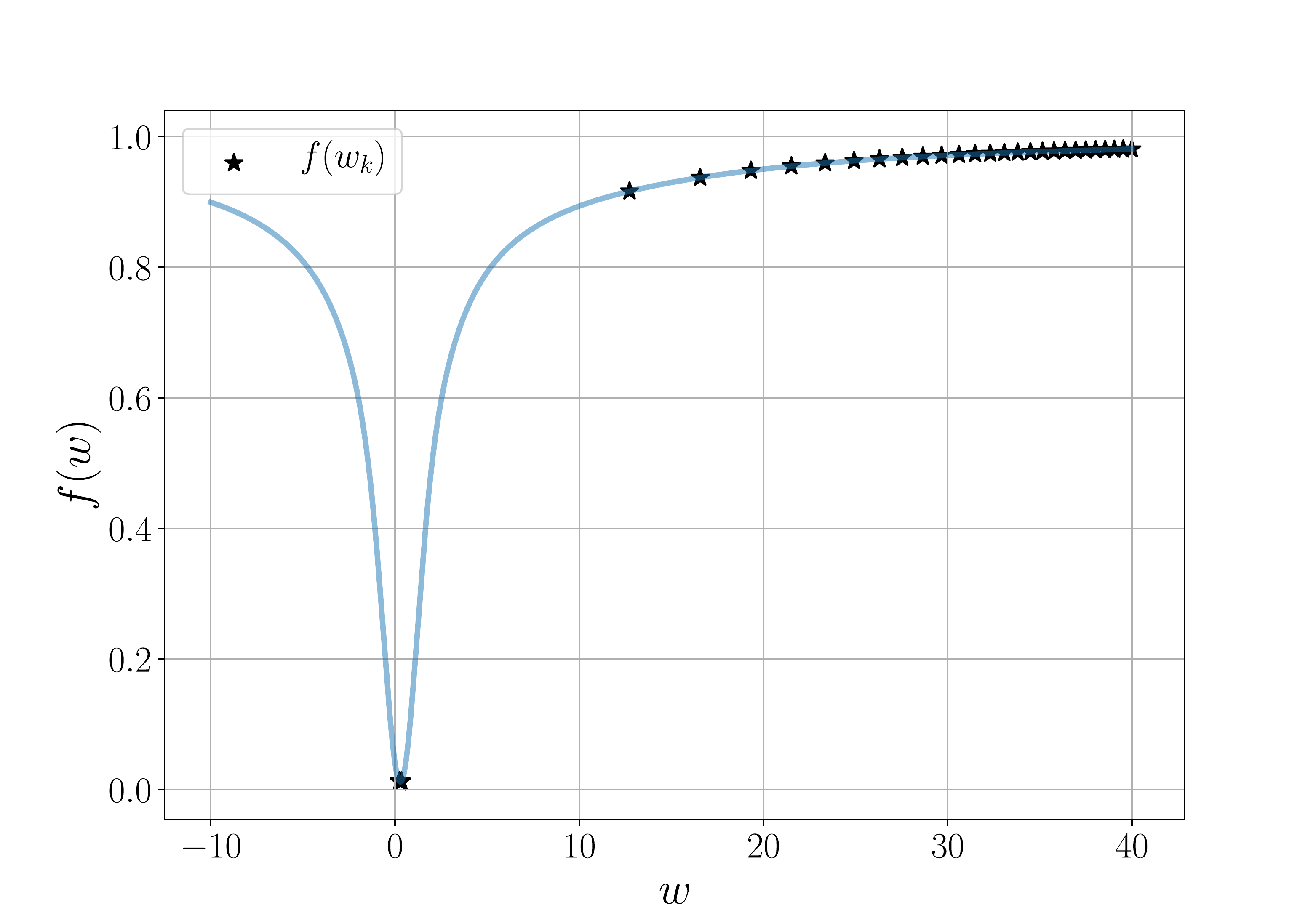}
    \caption{Visualization of the target objective functions vs $w$. The black stars represent $(w_k,\,f(w_k))$, i.e. the SGD iterates with associated objective function value.}\label{fig:toy_sgd_iterates}
  \end{minipage}
  %\caption{}
\end{figure}
\FloatBarrier

%\begin{wrapfigure}[8]{L}{0.5\textwidth}
\begin{figure}[!htbp]
\vspace{2cm} 
\centering
\begin{tikzpicture}
\draw[->, thick] (2.6,2) --  (3.4,2);
\draw[thick] (4,2) circle  (0.6cm) node {$x$};
\draw[->, thick] (4.6,2) --  (5.4,2);
\draw[thick] (6,2) circle (0.6cm)node {	$\mbox{\scriptsize\( %
\mathrm{erf}(w\cdot x) %
\)} $};
\draw[->, thick] (6.6,2) --  (7.4,2);
\draw[thick] (8,2) circle (0.6cm)node {$\hat{y}$};
\draw[->, thick] (8.6,2) --  (9.4,2);
\end{tikzpicture}
\caption{Graphical representation of the $1$-dimensional neural network deployed for the toy-case experiment described in Section~\ref{sec:toy_case}.} \label{fig:toy_net}
\end{figure}
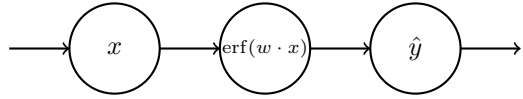
\FloatBarrier

%\end{wrapfigure} 
\begin{figure}[!htbp]
  \centering
  \begin{minipage}[t]{0.5\linewidth}
    \centering
    \includegraphics[width=\linewidth]{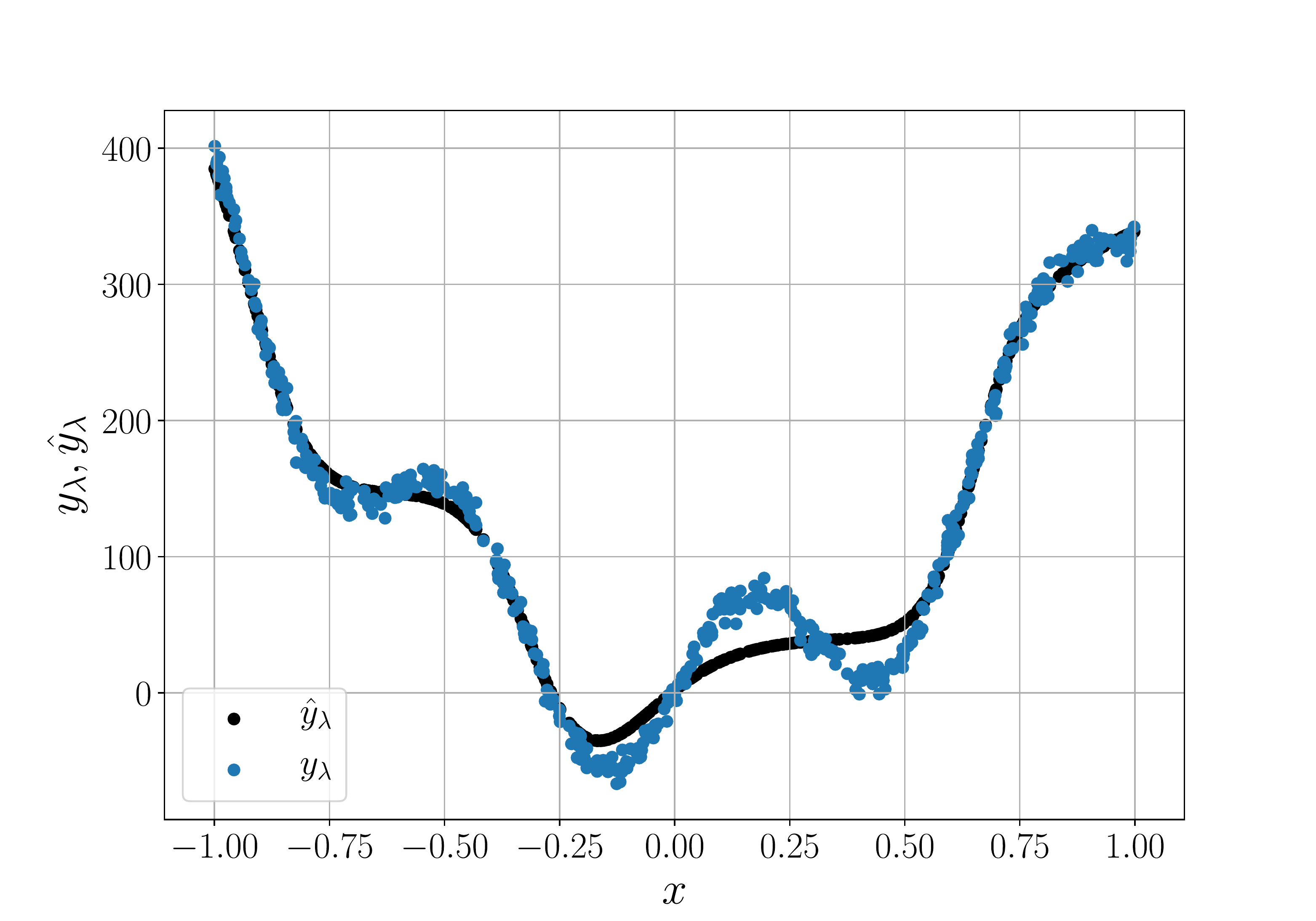} 
    \subcaption{$\lambda=0.63$} 
    \vspace{4ex}
  \end{minipage}%%
  \begin{minipage}[t]{0.5\linewidth}
    \centering
    \includegraphics[width=\linewidth]{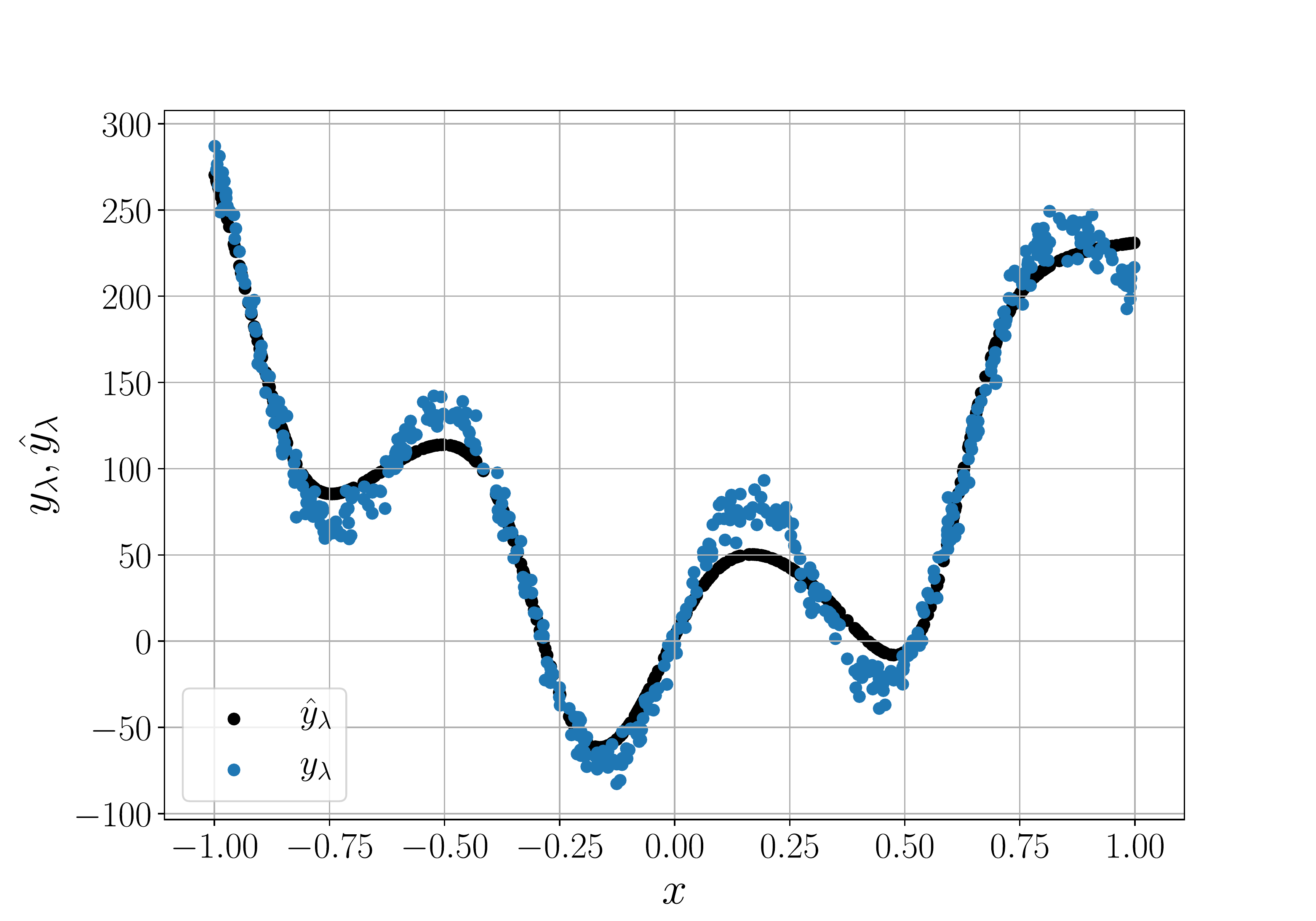} 
    \subcaption{$\lambda=0.75$} 
    \vspace{4ex}
  \end{minipage} 
  \begin{minipage}[t]{0.5\linewidth}
    \centering
    \includegraphics[width=\linewidth]{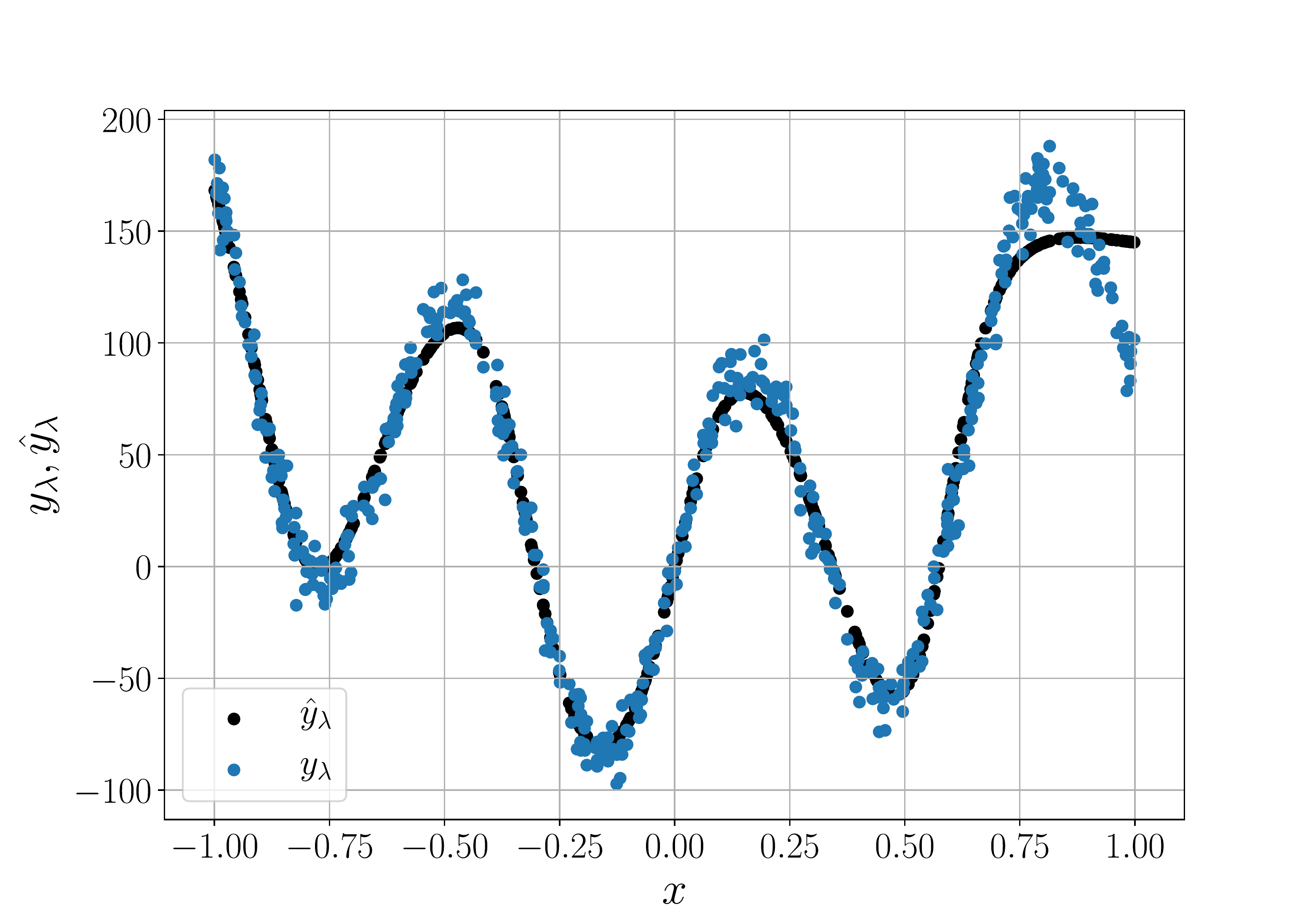} 
    \subcaption{$\lambda=0.86$} 
    \vspace{4ex}
  \end{minipage}%% 
  \begin{minipage}[t]{0.5\linewidth}
    \centering
    \includegraphics[width=\linewidth]{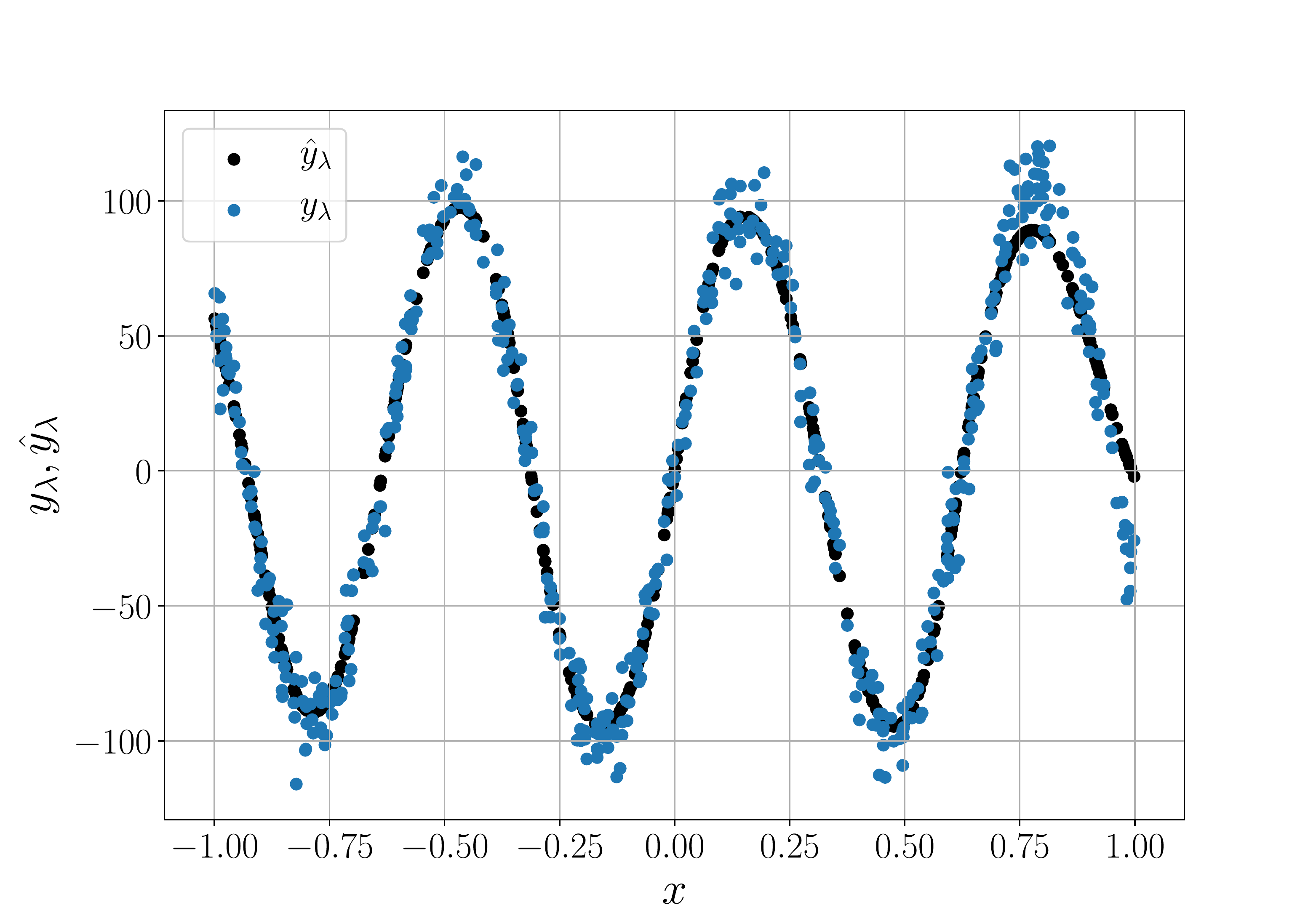} 
    \subcaption{$\lambda=0.98$} 
    \vspace{4ex}
  \end{minipage} 
  \caption{Predicted values $\hat{y}_{\lambda}$ (blue) vs true values $y_{\lambda}$ (black) for different values of $\lambda$ generated by using H-SGD.}\label{fig:hsgd_pred_lmbd}
\end{figure}
\FloatBarrier

\begin{figure}[!htbp]  
    \centering
    \includegraphics[width=\linewidth]{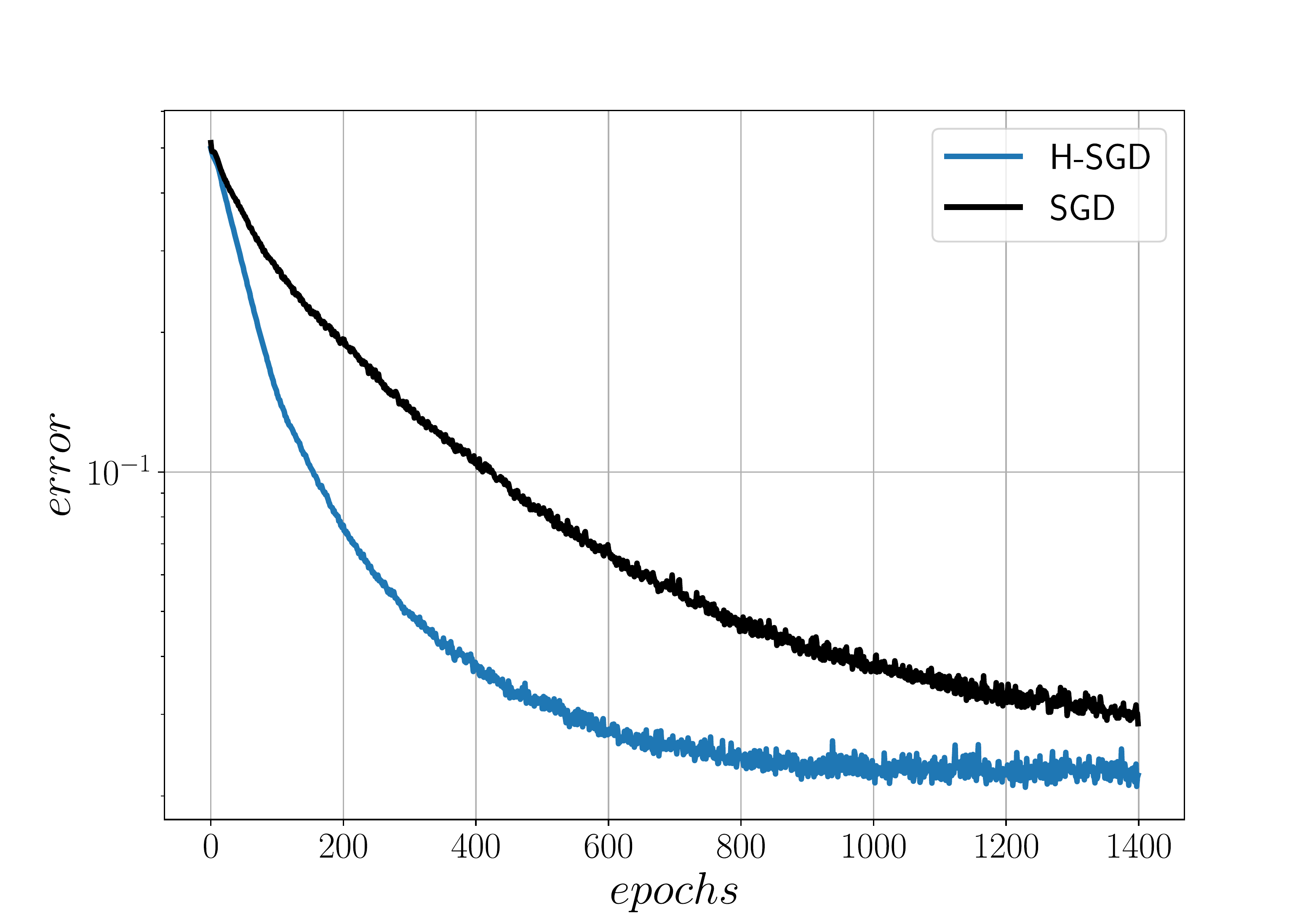} 
  	\caption{Error averaged across 100 runs of H-SGD (blue) and SGD (black) vs epochs for the classification task described in Section~\ref{sec:binary_classification}.}\label{fig:hsgd_sgd_classification}
\end{figure}
\FloatBarrier

\newpage
\section{Additional Remarks}\label{sec:additional_remarks}

\begin{remark}\label{app:remark}
Assumption~\ref{ass: regularity_1} can be read as a local Lipschitz continuity of $f$ with respect to its second argument.
Regarding the second regularity assumption (Assumption~\ref{ass: regularity_2}),~\eqref{eq:reg_2} relates the variation of the optimal objective function value   for the subset of solutions introduced in Assumption~\ref{ass: regularity_0} across homotopy iterations with the variations of the homotopy parameter. When the solution localization of the solution map is vector-valued, i.e. $w^*(\lambda)\equiv W^*(\lambda)$ (this is the case for instance when the Hessian of the objective function is positive definite at those points), this assumption can also be derived directly by combining Assumption~\ref{ass: regularity_1} with the following Lipschitz continuity requirements:
\begin{itemize}
\item Let $\kappa_1>0$, we assume that%Lipschitz continuity of $f(\cdot,\,\lambda)$ w.r.t. $w^*(\lambda_i)$ $\forall i$,
\begin{equation}\label{eq:requirement_1}
\vert f(w^*(\tilde{\lambda}),\,\hat{\lambda})- f(w^*(\hat{\lambda}),\,\hat{\lambda}) \vert \leq \kappa_1 \Vert w^*(\tilde{\lambda}) - w^*(\hat{\lambda}) \Vert,\quad \forall \tilde{\lambda},\,\hat{\lambda}\in [0,1]^z\,.
\end{equation}
\item Let $\kappa_2>0$, we assume that
\begin{equation}\label{eq:requirement_2}
\Vert w^*(\tilde{\lambda})-w^*(\hat{\lambda}) \Vert \leq \kappa_2\Vert \tilde{\lambda} - \hat{\lambda} \Vert,\quad \forall \tilde{\lambda},\,\hat{\lambda}\in[0,1]^z\,.
\end{equation}
\end{itemize}
In particular, by combining Inequalities~\eqref{eq:requirement_1} and~\eqref{eq:requirement_2}, we obtain
\begin{equation}\label{eq:requirement}
\vert f(w^*(\tilde{\lambda}),\,\hat{\lambda})- f(w^*(\hat{\lambda}),\,\hat{\lambda}) \vert \leq \kappa_2\,\kappa_1 \Vert \tilde{\lambda} - \hat{\lambda} \Vert,\quad \forall \tilde{\lambda},\,\hat{\lambda}\in [0,1]^z\,.
\end{equation}  
To recover Inequality~\eqref{eq:reg_2} where $\gamma\coloneqq \delta+ \kappa_1\,\kappa_2$, we use Inequalities~\eqref{eq:reg_1} and~\eqref{eq:requirement} together with the triangle inequality as follows
\begin{equation}
\begin{aligned}
\vert f(w^*(\tilde{\lambda}), \tilde{\lambda}) - f(w^*(\hat{\lambda}), \hat{\lambda}) \vert &= \vert f(w^*(\tilde{\lambda}), \tilde{\lambda}) - f(w^*(\tilde{\lambda}), \hat{\lambda}) + f(w^*(\tilde{\lambda}), \hat{\lambda}) - f(w^*(\hat{\lambda}), \hat{\lambda}) \vert\\
& \leq  \vert f(w^*(\tilde{\lambda}), \tilde{\lambda}) - f(w^*(\tilde{\lambda}), \hat{\lambda})\vert + \vert f(w^*(\tilde{\lambda}), \hat{\lambda}) - f(w^*(\hat{\lambda}), \hat{\lambda}) \vert\\
&\leq (\delta + \kappa_1\kappa_2)\Vert \tilde{\lambda}-\hat{\lambda} \Vert \,.
\end{aligned}
\end{equation}

Notice that the Assumption in~\eqref{eq:requirement_1} is a less restrictive condition of the following general Lipschitz continuity requirement
\begin{equation}
\vert f(v, \hat{\lambda}) - f(w, \hat{\lambda}) \vert \leq C \Vert v-w \Vert,\quad \forall v,\,w\in \mathbb{R}^d,\,\, \forall \hat{\lambda}\in[0,1]^z\,,
\end{equation}
where $C>0$.
\end{remark}

\begin{remark}\label{remark_PL}
Assumption~\ref{ass: PL_condition} is a more general version of the classical PL condition~\citep[see Section 2][for more details on the classical PL condition]{10.1007/978-3-319-46128-1_50}. In particular, it is straightforward to observe that the classical PL condition implies Assumption~\ref{ass: PL_condition}, but not vice versa. See Figure~\ref{fig:PL_condition} in Section~\ref{app:additional_figures} of the Appendix for a graphical representation of a one dimensional example with $w_{i-1,t}\sim\mathcal{N}(0,1)$ where Assumption~\ref{ass: PL_condition} holds, while the classical PL condition does not. The expected value operates a smoothing of the function landscape resulting in convexity, while the original function shows many bumps that make it non-convex.
\end{remark}

\section{Proof of Proposition 3.7}\label{sec:proof_proposition}
%proposition
\begin{proposition}\label{app_prop_1}
Consider $f(w,\,\lambda_i)$ with $\lambda_i \in [0,1]^z$ and let $w_{i-1,t}=w_{i-1,t}(\xi_{[i-1,t-1]})$ denote the iterate obtained at the $i$-th homotopy iteration by applying $t$ iterations of SGD with $t\leq k-1$ and $\alpha\leq \frac{1}{L}$. Under Assumptions~\ref{ass: L_smoothness},~\ref{ass: bounded_Var} and~\ref{ass: PL_condition}, if $\mathrm{E}_{\xi_{[i-1,t-1]}} \left[ f(w_{i-1,t}, \lambda_i) \right] -f^*(\lambda_i) \leq B$, then $\mathrm{E}_{\xi_{[i-1,t]}} \left[ f(w_{i-1,t+1}, \lambda_i) \right] -f^*(\lambda_i) \leq B$.
\end{proposition}
\begin{proof}
For ease of notation, we use the shorthands $w_{i-1,t}=w_{i-1,t}(\xi_{[i-1,t-1]})$ and $g_t= g(w_{i-1,t},\,\xi^i_t,\,\lambda_i)$.

Considering Assumption~\ref{ass: L_smoothness} together with the definition of SGD iterate and using Assumptions~\ref{ass: bounded_Var} and~\ref{ass: PL_condition}, we obtain the following inequalities
\begin{align*}
   & \mathrm{E}_{\xi_{[i-1,t]}} \left[ f(w_{i-1,t+1},\,\lambda_i) \right] \leq \mathrm{E}_{\xi_{[i-1,t]}}  \left[f(w_{i-1,t},\,\lambda_i) -\alpha \langle \nabla f(w_{i-1,t},\,\lambda_i), g_t \rangle + \frac{L\alpha^2}{2}\Vert g_t \Vert^2  \right] \\
   &\overset{\text{Law of Iterated Exp.}}{=}\mathrm{E}_{\xi_{[i-1,t-1]}}\left[ \mathrm{E}_{\xi^i_{t}}  \left[f(w_{i-1,t},\,\lambda_i) -\alpha \langle \nabla f(w_{i-1,t},\,\lambda_i), g_t \rangle + \frac{L\alpha^2}{2}\Vert g_t \Vert^2 \Big\vert \xi_{[i-1,t-1]} \right]\right] \\
    & \overset{\text{Assumption~\ref{ass: bounded_Var}}\quad}{\leq}  \mathrm{E}_{\xi_{[i-1,t-1]}} \left[ f(w_{i-1,t},\,\lambda_i) + \left(-\alpha+  \frac{L\alpha^2}{2} \right) \Vert \nabla f(w_{i-1,t},\,\lambda_i)\Vert^2  \right] + \frac{L\alpha^2\sigma^2}{2} \,.
\end{align*}
If $\alpha\leq \frac{1}{L}$, then
\begin{align}
    \mathrm{E}_{\xi_{[i-1,t]}} \left[ f(w_{i-1,t+1},\,\lambda_i)  \right] &\leq  \mathrm{E}_{\xi_{[i-1,t-1]}} \left[f(w_{i-1,t},\,\lambda_i)+ \left(-\frac{\alpha}{2} \right) \Vert \nabla f(w_{i-1,t}\,\lambda_i)\Vert^2  \right] + \frac{L\alpha^2\sigma^2}{2}\,.
\end{align}

We now make use of the ``expected'' PL condition and derive the following inequalities
\begin{equation}\label{eq:ineq_PL}
\begin{aligned}
&\mathrm{E}_{\xi_{[i-1,t]}}\left[ f(w_{i-1,t+1},\,\lambda_i) - f(w_{i-1,t},\,\lambda_i) \right] \leq -\frac{\alpha}{2}\mathrm{E}_{\xi_{[i-1,t-1]}}\left[  \Vert \nabla f(w_{i-1,t},\,\lambda_i) \Vert^2 \right] + \frac{L\alpha^2\sigma^2}{2} \\
&\overset{\text{Assumption~\ref{ass: PL_condition}}}{\leq} -\alpha\,\mu \left[  \mathrm{E}_{\xi_{[i-1,t-1]}}\left[  f(w_{i-1,t},\,\lambda_i) \right] - f^*(\lambda_i) \right] + \frac{L\alpha^2\sigma^2}{2}\,.\\ 
\end{aligned}
\end{equation}

From Inequality~\eqref{eq:ineq_PL} it follows that, whenever $\mathrm{E}_{\xi_{[i-1,t-1]}}\left[ f(w_{i-1,t},\,\lambda_i) \right] - f^*(\lambda_i) \geq \frac{\sigma^2}{2\mu}$, the objective function decreases in expectation, i.e. $\mathrm{E}_{\xi_{[i-1,t]}}\left[ f(w_{i-1,t+1},\,\lambda_i) - f(w_{i-1,t},\,\lambda_i) \right]\leq 0$. Given that by assumption $B> \frac{\sigma^2}{2\mu}$, we can consequently conclude that, if $\mathrm{E}_{\xi_{[i-1,t-1]}} \left[ f(w_{i-1,t}, \lambda_i) \right] -f(w_i^*, \lambda_i) \leq B$, then $\mathrm{E}_{\xi_{[i-1,t]}} \left[ f(w_{i-1,t+1}, \lambda_i) \right] -f^*(\lambda_i) \leq B$.
\end{proof}

\section{Proof of Theorem 3.8}\label{sec:proof_th_sgd}
%SGD error bounds
\begin{theorem}\label{th_app_1}
Consider the minimization of $f(w,\,\lambda_i)$ with $\lambda_i \in [0,1]^z$ via SGD. Let $w_{i-1}=w_{i-1}(\xi_{[i-1]})$ be the random initial point associated with the $i$-th homotopy iteration with $\mathbb{E}_{\xi_{[i-1]}} \left[f(w_{i-1}, \lambda_i) \right] -f^*(\lambda_i) \leq B$ and $w_{i-1,t}=w_{i-1,t}(\xi_{[i-1,t-1]})$ denote the $t$-th SGD iterate with $t\leq k$. Under Assumptions~\ref{ass: L_smoothness},~\ref{ass: bounded_Var} and~\ref{ass: PL_condition}, SGD with a constant step-size $\alpha\leq \frac{1}{L}$ attains the following convergence rate to a minimizer's neighborhood
\begin{equation}\label{eq:convergence_SGD}
    \mathrm{E}_{\xi_{[i-1,t-1]}}\left[f(w_{i-1,t},\,\lambda_i)-f^*(\lambda_i) \right]\leq \rho^t \mathrm{E}_{\xi_{[i-1]}}\left[f(w_{i-1},\,\lambda_i)-f^*(\lambda_i) \right] + \frac{\sigma^2}{2\mu}\,, 
\end{equation}
with $\rho\coloneqq\left(1-\alpha\mu \right)$. 
With $\alpha=\frac{1}{L}$, we obtain $\rho=\left(1-\frac{\mu}{L} \right)$.
\end{theorem}
\begin{proof}
For ease of notation, we use the shorthands $w_{i-1,t} = w_{i-1,t}(\xi_{[i-1,t-1]})$ and $g_t = g(w_{i-1,t},\,\xi^i_t,\,\lambda_i)$.
%Recall that, as discussed in Section~\ref{sec:intro}, the iterate $w_{t+1}$ is a function of the history $w_{t+1}(\xi^{[i-1]}_{[k-1]},\xi^{[i]}_{[t]})$. 

By combining Assumption~\ref{ass: L_smoothness} with the definition of SGD iterate, for all $t=1,\dots,k-1$, we obtain the following bound on $f(w_{t+1},\,\lambda_i)$ 
\begin{align*}
    &f(w_{i-1,t+1},\,\lambda_i)\leq  f(w_{i-1,t},\,\lambda_i) + \langle\nabla f(w_{i-1,t},\,\lambda_i), w_{i-1,t+1}-w_{i-1,t} \rangle + \frac{L}{2} \Vert w_{i-1,t+1}-w_{i-1,t}\Vert^2\\
    & \overset{w_{i-1,t+1} \coloneqq w_{i-1,t} -\alpha g_t}{=} f(w_{i-1,t},\,\lambda_i)-\alpha \langle \nabla f(w_{i-1,t},\,\lambda_i), g_t \rangle + \frac{L\alpha^2}{2}\Vert g_t \Vert^2 \,.
\end{align*}
We now take the expectation with respect to all the sources of randomness involved and then we apply the law of iterated expectations, together with Assumption~\ref{ass: bounded_Var}
\begin{align*}
   & \mathrm{E}_{\xi_{[i-1,t]}} \left[ f(w_{i-1,t+1},\,\lambda_i) \right] \leq \mathrm{E}_{\xi_{[i-1,t]}}  \left[f(w_{i-1,t},\,\lambda_i) -\alpha \langle \nabla f(w_{i-1,t},\,\lambda_i), g_t \rangle + \frac{L\alpha^2}{2}\Vert g_t \Vert^2  \right] \\
   &\overset{\text{Law of Iterated Exp.}}{=}\mathrm{E}_{\xi_{[i-1,t-1]}}\left[ \mathrm{E}_{\xi^i_{t}}  \left[f(w_{i-1,t},\,\lambda_i) -\alpha \langle \nabla f(w_{i-1,t},\,\lambda_i), g_t \rangle + \frac{L\alpha^2}{2}\Vert g_t \Vert^2 \Big\vert \xi_{[i-1,t-1]}\right]\right] \\
    & \overset{\text{\quad Assumption~\ref{ass: bounded_Var}}}{\leq}  \mathrm{E}_{\xi_{[i-1,t-1]}} \left[ f(w_{i-1,t},\,\lambda_i) + \left(-\alpha+  \frac{L\alpha^2}{2} \right) \Vert \nabla f(w_{i-1,t},\,\lambda_i)\Vert^2  \right] + \frac{L\alpha^2\sigma^2}{2} \,.
\end{align*}
If $\alpha\leq \frac{1}{L}$, then
\begin{align}\label{eq:expected_decrease}
    \mathrm{E}_{\xi_{[i-1,t]}} \left[ f(w_{i-1,t+1},\,\lambda_i)  \right] &\leq  \mathrm{E}_{\xi_{[i-1,t-1]}} \left[f(w_{i-1,t},\,\lambda_i)+ \left(-\frac{\alpha}{2} \right) \Vert \nabla f(w_{i-1,t},\,\lambda_i)\Vert^2  \right] + \frac{L\alpha^2\sigma^2}{2}\,.
\end{align}
We now apply the PL condition to Inequality~\eqref{eq:expected_decrease}, and we obtain
\begin{equation}
\begin{aligned}
    \mathrm{E}_{\xi_{[i-1,t]}} \left[ f(w_{i-1,t+1},\,\lambda_i) \right] &\overset{\text{Assumption~\ref{ass: PL_condition}}}{\leq}  \mathrm{E}_{\xi_{[i-1,t-1]}} \left[f(w_{i-1,t},\,\lambda_i)\right.\\& \left. -\alpha\mu \right( f(w_{i-1,t},\,\lambda_i) - f^*(\lambda_i) \left)   \right] + \frac{L\alpha^2\sigma^2}{2}\,.
\end{aligned}
\end{equation}

By subtracting $ f^*(\lambda_i) $ on both sides and setting $\alpha=\frac{1}{L}$, we obtain the following inequality
\begin{align}\label{eq:ineq_recur}
    \mathrm{E}_{\xi_{[i-1,t]}} \left[ f(w_{i-1,t+1},\,\lambda_i) - f^*(\lambda_i) \right] &\leq \left(1-\frac{\mu}{L} \right)  \mathrm{E}_{\xi_{[i-1,t-1]}} \left[  f(w_{i-1,t},\,\lambda_i) - f^*(\lambda_i)    \right] + \frac{\sigma^2}{2L}\,.
\end{align}
By applying Inequality~\eqref{eq:ineq_recur} recursively, we derive the following bound
\begin{equation}
\begin{aligned}
    \mathrm{E}_{\xi_{[i-1,t-1]}} \left[ f(w_{i-1,t},\,\lambda_i) - f^*(\lambda_i) \right] &\leq \left(1-\frac{\mu}{L} \right)^k \mathrm{E}_{\xi_{[i-1]}} \left[   f(w_{i-1},\,\lambda_i) - f^*(\lambda_i)    \right]\\
    & \quad\quad\quad + \frac{\sigma^2}{2L}\sum_{j=0}^{k-1} \left(1-\frac{\mu}{L} \right)^j\,.
\end{aligned}
\end{equation}
Finally, by using the limit of geometric series, we obtain
\begin{align}
    \mathrm{E}_{\xi_{[i-1,t-1]}} \left[ f(w_{i-1,t},\,\lambda_i) - f^*(\lambda_i) \right] &\leq \left(1-\frac{\mu}{L} \right)^k \mathrm{E}_{\xi_{[i-1]}} \left[   f(w_{i-1},\,\lambda_i) - f^*(\lambda_i)    \right] + \frac{\sigma^2}{2\mu}\,.
\end{align}
\end{proof}

\section{Proof of Lemma 3.9}\label{sec:proof_lemma}
%lemma
\begin{lemma}\label{app_lemma_1}
Assume $\Vert \lambda_{i+1} - \lambda_i \Vert \leq \epsilon$, $0\leq \epsilon < \frac{B}{\delta + \gamma}$ and let $w_i$ denote the $i$-th iterate of Algorithm~\ref{alg:hsgd} with $\alpha\leq \frac{1}{L}$. Under Assumptions~\ref{ass: L_smoothness}-~\ref{ass: regularity_2} and~\ref{ass: bounded_Var}-~\ref{ass: PL_condition}, if  $\phi_{w_i}(\lambda_i)\leq B-(\delta+\gamma)\epsilon$, then $\phi_{w_i}(\lambda_{i+1})\leq B$.
In addition, let 
\begin{equation}
k_{\max} \coloneqq \Bigg\lceil\log_{\rho}\left( 1 - \frac{2\mu(\delta + \gamma)\epsilon + \sigma^2}{2\mu B} \right)\Bigg\rceil\,.
\end{equation}  
If $\phi_{w_i}(\lambda_{i+1})\leq B$, $0\leq \epsilon< \frac{1}{\delta + \gamma}\left(B - \frac{\sigma^2}{2\mu} \right)$ and $k\geq k_{\max}$, then $\phi_{w_{i+1}}(\lambda_{i+1})\leq B-(\delta+\gamma)\epsilon$.
\end{lemma}
\begin{proof}
We start by deriving an upper bound on $\mathrm{E}_{\xi_{[i]}}\left[f(w_i,\,\lambda_{i+1}) \right]-f^*(\lambda_{i+1})$ with $\phi_{w_i}(\lambda_i)\leq B$ and $\Vert \lambda_{i+1}-\lambda_i \Vert\leq \epsilon$. For that we use the regularity Assumptions~\ref{ass: regularity_1} and~\ref{ass: regularity_2} together with the triangle and Jensen inequalities as follows
\begin{equation}\label{eq:ineq_1}
\begin{aligned}
& \mathrm{E}_{\xi_{[i]}}\left[f(w_i,\,\lambda_{i+1}) \right]-f^*(\lambda_{i+1}) = \vert\mathrm{E}_{\xi_{[i]}}\left[f(w_i,\,\lambda_{i+1}) \right]-f^*(\lambda_{i+1})\vert \\
& \overset{\phantom{\text{Triangle and Jensen Ineq.}}}{=}\vert\mathrm{E}_{\xi_{[i]}}\left[f(w_i,\,\lambda_{i+1}) + f(w_i,\,\lambda_i) - f(w_i,\,\lambda_i) \right]-f^*(\lambda_{i+1}) \\
& \phantom{\overset{\text{Triangle and Jensen Ineq.}}{=}} + f^*(\lambda_i) - f^*(\lambda_i)\vert\\
&\overset{\text{Triangle and Jensen Ineq.}}{\leq}\vert\mathrm{E}_{\xi_{[i]}}\left[ f(w_i,\,\lambda_i) \right] - f^*(\lambda_i)  \vert + \mathrm{E}_{\xi_{[i]}}\left[ \vert f(w_i,\,\lambda_{i+1})  - f(w_i,\,\lambda_i)\vert\right]\\
&\phantom{\overset{\text{Triangle and Jensen Ineq.}}{\leq}} + \vert f^*(\lambda_i) - f^*(\lambda_{i+1}) \vert\\
&\overset{\text{Assumptions~\ref{ass: regularity_1} and~\ref{ass: regularity_2}}}{\leq} \vert\mathrm{E}_{\xi_{[i]}}\left[f(w_i,\,\lambda_i) \right] - f^*(\lambda_i)\vert + (\delta+\gamma)\epsilon\,.
\end{aligned}
\end{equation}
From Inequality~\eqref{eq:ineq_1} it follows that, if $\phi_{w_i}(\lambda_i)\leq B-(\delta+\gamma)\epsilon$ with $\epsilon< \frac{B}{(\delta+\gamma)}$, then $\phi_{w_i}(\lambda_{i+1})\leq B$.

We now use the results of Theorem~\ref{theorem_SGD} to derive a lower bound on the number of SGD-steps such that, if $\phi_{w_i}(\lambda_{i+1})\leq B$, then $\phi_{w_{i+1}}(\lambda_{i+1})\leq B-(\delta+\gamma)\epsilon$. We start considering the following inequality

\begin{equation}\label{eq:ineq_2}
\begin{aligned}
\mathrm{E}_{\xi_{[i+1]}}\left[ f(w_{i+1},\,\lambda_{i+1}) \right] - f^*(\lambda_{i+1}) &\leq \rho^k \left[ \mathrm{E}_{\xi_{[i]}}\left[ f(w_i,\,\lambda_{i+1}) \right] - f^*(\lambda_{i+1}) \right] + \frac{\sigma^2}{2\mu}\\
&\leq \rho^k B + \frac{\sigma^2}{2\mu}\,.
\end{aligned}
\end{equation}
From Inequality~\eqref{eq:ineq_2} it follows that, if $\phi_{w_i}(\lambda_{i+1})\leq B$, then $\phi_{w_{i+1}}(\lambda_{i+1})\leq B-(\delta+\gamma)\epsilon$ with $\epsilon< \frac{1}{\delta+\gamma}\left(B-\frac{\sigma^2}{2\mu} \right)$ whenever
\begin{equation}\label{eq:ineq_3}
k \geq \Bigg\lceil\log_{\rho}\left( 1 - \frac{2\mu(\delta + \gamma)\epsilon + \sigma^2}{2\mu B} \right)\Bigg\rceil\,.
\end{equation} 
\end{proof}
%final theorem for optimality tracking
\section{Proof of Theorem 3.10}\label{sec:proof_th_ot}
\begin{theorem}\label{app_th_2} 
Assume there exists $\frac{\sigma^2}{2\mu}< r\leq B$ and $\tilde{\epsilon} \coloneqq\min \left\{ \epsilon_1,\,\epsilon_2\right\}$ with
\begin{equation}
\epsilon_1\coloneqq \frac{1}{(\delta+\gamma)}(B-r),\quad\epsilon_2 \coloneqq \frac{(1-\rho^k)\, r-\sigma^2/2\mu}{\rho^k\,(\delta+\gamma)}\,.
\end{equation}

In addition, let
\begin{equation}
k_{\max}\coloneqq \Bigg\lceil \log_{\rho} \left( 1-\frac{\sigma^2}{2\mu r} \right)  \Bigg\rceil\,.
\end{equation}

Consider Algorithm~\ref{alg:hsgd} with $\alpha \leq \frac{1}{L}$, $k\geq k_{\max}$ and $\Vert \lambda_{i+1} - \lambda_i \Vert \leq \epsilon$, where $0\leq\epsilon\leq \tilde{\epsilon}$.

Under Assumptions~\ref{ass: L_smoothness}-~\ref{ass: regularity_2} and~\ref{ass: bounded_Var}-~\ref{ass: PL_condition}, if $\phi_{w_i}(\lambda_i)\leq r$, then $\phi_{w_{i+1}}(\lambda_{i+1})\leq r$.
\end{theorem}

\begin{proof}
We consider $\phi_{w_i}(\lambda_i)\leq r$. If $\epsilon\leq \frac{1}{(\delta + \gamma)}(B-r)$ with $r\leq B$, then $\phi_{w_i}(\lambda_i)\leq B-(\delta+\gamma)\epsilon$ and, as shown in Lemma~\ref{lemma:initial_conditions}, $\phi_{w_i}(\lambda_{i+1})\leq B$. This allows us to use the results of Theorem~\ref{theorem_SGD}.

We now derive an upper bound on $\mathrm{E}_{\xi_{[i+1]}}\left[ f(w_{i+1},\,\lambda_{i+1}) \right] - f^*(\lambda_{i+1})$ by considering the results of Theorem~\ref{theorem_SGD} together with the regularity Assumptions~\ref{ass: regularity_1} and~\ref{ass: regularity_2}, and the triangle and Jensen inequalities as follows
\begin{equation}
\begin{aligned}
&\mathrm{E}_{\xi_{[i+1]}}\left[ f(w_{i+1},\,\lambda_{i+1}) \right] - f^*(\lambda_{i+1}) \leq \rho^k\left[ \mathrm{E}_{\xi_{[i]}}\left[ f(w_i,\,\lambda_{i+1}) \right] - f^*(\lambda_{i+1})   \right] + \frac{\sigma^2}{2\mu}\\
&\overset{\phantom{\text{Triangle and Jensen Ineq.}}}{=} \rho^k \Big\vert  \mathrm{E}_{\xi_{[i]}}\left[ f(w_i,\,\lambda_{i+1}) + f(w_i,\,\lambda_i) - f(w_i,\,\lambda_i) \right] - f^*(\lambda_{i+1}) \\
&\phantom{\overset{\text{Triangle and Jensen Ineq.}}{=}} + f^*(\lambda_i) - f^*(\lambda_i)   \Big\vert + \frac{\sigma^2}{2\mu}  \\
&\overset{\text{Triangle and Jensen Ineq.}}{\leq} \rho^k \left[ \mathrm{E}_{\xi_{[i]}}\left[ f(w_i,\,\lambda_i)  \right] - f^*(\lambda_i)  \right] + \rho^k \mathrm{E}_{\xi_{[i]}}\left[\vert f(w_i,\,\lambda_{i+1}) - f(w_{i},\,\lambda_{i})  \vert  \right]   \\
&\phantom{\overset{\text{Triangle and Jensen Ineq.}}{=}\,}+ \rho^k\vert f^*(\lambda_i) - f^*(\lambda_{i+1}) \vert + \frac{\sigma^2}{2\mu}  \\
&\overset{\text{Assumptions~\ref{ass: regularity_1} and~\ref{ass: regularity_2}}}{\leq} \rho^k \left[ \mathrm{E}_{\xi_{[i]}}\left[ f(w_i,\,\lambda_i)  \right] - f^*(\lambda_i)  \right] + \rho^k (\delta + \gamma) \Vert \lambda_{i+1}-\lambda_i \Vert + \frac{\sigma^2}{2\mu}\,. \\
\end{aligned}
\end{equation}
Using the fact that $\phi_{w_i}(\lambda_i)\leq r$ and that $\Vert \lambda_{i+1}-\lambda_i \Vert\leq \epsilon$, we now solve the following inequality for $\epsilon$ in order to find an upper bound on the variation of the homotopy parameter such that $\phi_{w_{i+1}}(\lambda_{i+1})\leq r $
\begin{equation}\label{eq:ineq_4}
\rho^k\, r + \rho^k\,(\delta + \gamma)\,\epsilon + \frac{\sigma^2}{2\mu}\,{\leq}\, r\,.
\end{equation} 
Inequality~\eqref{eq:ineq_4} holds whenever
\begin{equation}
k \geq \Bigg\lceil \log_{\rho} \left( 1-\frac{\sigma^2}{2\mu r} \right)  \Bigg\rceil \,,
\end{equation}
and
\begin{equation}
\epsilon \leq \frac{(1-\rho^k)\,r - \sigma^2/2\mu}{\rho^k\,(\delta + \gamma)}\,,
\end{equation}
with $r > \frac{\sigma^2}{2\mu}$.
%Consequently, we obtain that $\frac{\sigma^2}{2\mu(1-\rho^k)}\leq r \leq B$. Notice that in the considered setting, i.e. $k\geq k_{max}$ and $B>\frac{\sigma^2}{2\mu}$, we have that $\frac{\sigma^2}{2\mu(1-\rho^k)}\leq B$. 
\end{proof}
\section{Proof of Theorem 3.11}\label{sec:proof_th_lc}
\begin{theorem}\label{th:app_lc}
Let $\tilde{\rho}\in\left( 1-\frac{\sigma^2}{2\mu}\frac{1}{B},1 \right)$ and consider Algorithm~\ref{alg:hsgd} with $\alpha\leq \frac{1}{L}$, $\phi_{w_0}(\lambda_0)\leq r$ with $\frac{\sigma^2}{2\mu}\frac{1}{(1-\tilde{\rho})}\leq r\leq B $ and  $k\geq \log_{\rho}(\tilde{\rho})$. In addition, let $\epsilon_1\coloneqq \frac{1}{(\delta+\gamma)}(B-r)$ and 

\begin{equation}
C_{\tilde{\rho}}\coloneqq
\begin{cases}
      1 & \text{if }k\geq \log_{\rho}(\tilde{\rho})-\log_{\rho}\left( 1 + \frac{\delta+\gamma}{\varepsilon_0} \right)\\
      \frac{\tilde{\rho}-\rho^k}{\rho^k}\frac{\varepsilon_0}{(\delta+\gamma)} & \text{otherwise,}
    \end{cases} 
\end{equation}
with $\varepsilon_0\coloneqq \mathrm{E}_{\xi_{[0]}}\left[ f(w_0,\,\lambda_0) \right] - f^*(\lambda_0)$.

Under Assumptions~\ref{ass: L_smoothness}-~\ref{ass: regularity_2} and~\ref{ass: bounded_Var}-~\ref{ass: PL_condition}, if $\Vert \lambda_{i+1}-\lambda_i \Vert\leq \min\left\{e^{-\eta\,i}, \epsilon_1\right\}$ with $\eta\geq \ln{\left(C_{\tilde{\rho}}\,\tilde{\rho}\right)}$, then
\begin{equation}
\mathrm{E}_{\xi_{[i+1]}}\left[ f(w_{i+1},\,\lambda_{i+1}) \right] - f^*(\lambda_{i+1}) \leq \tilde{\rho}^{i+1}\left[ \mathrm{E}_{\xi_{[0]}}\left[  f(w_0, \,\lambda_0) \right] - f^*(\lambda_0) \right]  + \frac{\sigma^2}{2\mu}\sum_{j=0}^i \tilde{\rho}^j\,.
\end{equation}
\end{theorem}

\begin{proof}
We start assuming that $\phi_{w_i}(\lambda_i)\leq r$, with $0\leq r \leq B$ and $\Vert \lambda_{i+1}-\lambda_i \Vert\leq \epsilon_1$ with $\epsilon_1\coloneqq \frac{1}{(\delta+\gamma)}\left( B-r \right)$ such that $\phi_{w_i}(\lambda_{i+1})\leq B$.

In particular, we consider the following upper bound on $\mathrm{E}_{\xi_{[i+1]}}\left[ f(w_{i+1},\,\lambda_{i+1}) \right] - f^*(\lambda_{i+1})$,
\begin{equation}\label{eq:upper_bound}
\begin{aligned}
\mathrm{E}_{\xi_{[i+1]}}\left[ f(w_{i+1},\,\lambda_{i+1}) \right] - f^*(\lambda_{i+1}) &\leq \rho^k\left[\mathrm{E}_{\xi_{[i]}}\left[ f(w_{i},\,\lambda_{i}) \right] - f^*(\lambda_{i})\right]\\&\quad + \rho^k\,(\delta+\gamma)\,\Delta\lambda_{i+1} + \frac{\sigma^2}{2\mu}\,.
\end{aligned}
\end{equation} 
See the the proof of Theorem~\ref{theorem_1} for a derivation.

We then proceed by induction. Therefore, we assume 
\begin{equation}\label{eq:ass_induction}
\mathrm{E}_{\xi_{[i]}}\left[ f(w_{i},\,\lambda_{i}) \right] - f^*(\lambda_{i}) \leq \tilde{\rho}^i \left[ \mathrm{E}_{\xi_{[0]}}\left[ f(w_{0},\,\lambda_{0}) \right] - f^*(\lambda_{0})  \right] + \frac{\sigma^2}{2\mu}\sum_{j=0}^{i-1}\tilde{\rho}^j\,, 
\end{equation}
and derive the conditions on $\Delta\lambda_{i+1}$ such that
\begin{equation}
\mathrm{E}_{\xi_{[i+1]}}\left[ f(w_{i+1},\,\lambda_{i+1}) \right] - f^*(\lambda_{i+1}) \leq \tilde{\rho}^{i+1} \left[ \mathrm{E}_{\xi_{[0]}}\left[ f(w_{0},\,\lambda_{0}) \right] - f^*(\lambda_{0})  \right] + \frac{\sigma^2}{2\mu}\sum_{j=0}^{i}\tilde{\rho}^j\,. 
\end{equation}

In order to achieve that, we consider the upper bound on $\mathrm{E}_{\xi_{[i+1]}}\left[ f(w_{i+1},\,\lambda_{i+1}) \right] - f^*(\lambda_{i+1}) $ given by Inequality~\eqref{eq:upper_bound} and solve the following inequality for $\Delta\lambda_{i+1}$
\begin{equation}\label{eq:ineq_5}
\begin{aligned}
&\rho^k\left[\mathrm{E}_{\xi_{[i]}}\left[ f(w_{i},\,\lambda_{i}) \right] - f^*(\lambda_{i})\right] + \rho^k\,(\delta+\gamma)\,\Delta\lambda_{i+1} + \frac{\sigma^2}{2\mu}  \\&\leq \tilde{\rho}^{i+1} \left[ \mathrm{E}_{\xi_{[0]}}\left[ f(w_{0},\,\lambda_{0}) \right] - f^*(\lambda_{0})  \right] + \frac{\sigma^2}{2\mu}\sum_{j=0}^{i}\tilde{\rho}^j\,.
\end{aligned}
\end{equation}
We obtain that Inequality~\eqref{eq:ineq_5} is satisfied whenever
\begin{equation}\label{eq:sol_delta_lmbd}\small
\Delta\lambda_{i+1}\leq \underbrace{\frac{\tilde{\rho}^{i+1}\left[ \mathrm{E}_{\xi_{[0]}}\left[ f(w_{0},\,\lambda_{0}) \right] - f^*(\lambda_{0})  \right] - \rho^k \left[ \mathrm{E}_{\xi_{[i]}}\left[ f(w_{i},\,\lambda_{i}) \right] - f^*(\lambda_{i})  \right] + \frac{\sigma^2}{2\mu}\sum_{j=1}^{i}\tilde{\rho}^j}{\rho^k (\delta+\gamma)}}_{\text{RHS}}\,.
\end{equation}
 
We derive a lower bound on the right-had side of~\eqref{eq:sol_delta_lmbd} by considering the induction assumption, i.e. Inequality~\eqref{eq:ass_induction}, and we obtain
\begin{equation}
\text{RHS}\geq \frac{\tilde{\rho}^{i}\left(\tilde{\rho}-\rho^k \right)\left[ \mathrm{E}_{\xi_{[0]}}\left[ f(w_{0},\,\lambda_{0}) \right] - f^*(\lambda_{0})  \right] - \frac{\sigma^2}{2\mu}\rho^k\sum_{j=0}^{i-1}\tilde{\rho}^j +\frac{\sigma^2}{2\mu}\sum_{j=1}^{i}\tilde{\rho}^j}{\rho^k (\delta+\gamma)} \,.
\end{equation}
Considering that $k\geq \log_{\rho}(\tilde{\rho})$ then
\begin{equation}
- \frac{\sigma^2}{2\mu}\rho^k\sum_{j=0}^{i-1}\tilde{\rho}^j +\frac{\sigma^2}{2\mu}\sum_{j=1}^{i}\tilde{\rho}^j = \frac{\sigma^2}{2\mu}\left(1-\frac{\rho^k}{\tilde{\rho}} \right)\sum_{j=0}^{i-1} \tilde{\rho}^j \geq 0\,. 
\end{equation}
Consequently, Inequality~\eqref{eq:sol_delta_lmbd} is satisfied whenever
\begin{equation}\label{eq:final_lmbd_cond}
\Delta\lambda_{i+1}\leq \frac{\left(\tilde{\rho}-\rho^k \right)}{\rho^k}\frac{\varepsilon_0}{ (\delta+\gamma)}\tilde{\rho}^{i}\,,
\end{equation}
where $\varepsilon_0\coloneqq\mathrm{E}_{\xi_{[0]}}\left[ f(w_{0},\,\lambda_{0}) \right] - f^*(\lambda_{0}) $.

To conclude this first part of derivations, we obtain that, whenever Inequality~\eqref{eq:final_lmbd_cond} is satisfied, then
\begin{equation}\label{eq:final_ineq}
 \mathrm{E}_{\xi_{[i+1]}}\left[ f(w_{i+1},\,\lambda_{i+1}) \right] - f^*(\lambda_{i+1})   \leq \tilde{\rho}^{i+1}\left(  \mathrm{E}_{\xi_{[0]}}\left[ f(w_{0},\,\lambda_{0}) \right] - f^*(\lambda_{0})   \right)   +  \frac{\sigma^2}{2\mu}\sum_{j=0}^{i}\tilde{\rho}^j\,.
\end{equation}

In order to ensure that $\phi_{w_{i+1}}(\lambda_{i+1})\leq r$ we consider Inequality~\eqref{eq:final_ineq} and use the fact that $\phi_{w_0}(\lambda_0)\leq r$, i.e. $ \mathrm{E}_{\xi_{[0]}}\left[ f(w_{0},\,\lambda_{0}) \right] - f^*(\lambda_{0}) \leq r$, to upper bound the right-hand side of Inequality~\eqref{eq:final_ineq}. We then solve the resulting inequality for $r$
\begin{equation}\label{eq:r_ineq}
\tilde{\rho}^{i+1} r + \frac{\sigma^2}{2\mu}\sum_{j=0}^{i}\tilde{\rho}^j\leq r\,.
\end{equation}

Considering that $\sum_{j=0}^{i}\tilde{\rho}^j=\frac{\left( 1-\tilde{\rho}^{i+1} \right)}{\left(1-\tilde{\rho} \right)}$, we obtain that Inequality~\eqref{eq:r_ineq} is satisfied whenever 
\begin{equation}\label{eq:r_sol}
r\geq \frac{\sigma^2}{2\mu}\frac{1}{\left( 1-\tilde{\rho} \right)}\,.
\end{equation}
By combining Inequality~\eqref{eq:r_sol} with the fact that $r\leq B$, we obtain the following upper bound on $\tilde{\rho}$
\begin{equation}
\tilde{\rho}\leq 1-\frac{\sigma^2}{2\mu}\frac{1}{B}\,.
\end{equation}

To further simplify the bound in~\eqref{eq:final_lmbd_cond}, we define the following constant
\begin{equation}
C_{\tilde{\rho}}\coloneqq
\begin{cases}
      1 & \text{if }k\geq \log_{\rho}(\tilde{\rho})-\log_{\rho}\left( 1 + \frac{\delta+\gamma}{\varepsilon_0} \right)\\
      \frac{\tilde{\rho}-\rho^k}{\rho^k}\frac{\varepsilon_0}{(\delta+\gamma)} & \text{otherwise,}
    \end{cases} 
\end{equation}
and obtain that Inequality~\eqref{eq:final_lmbd_cond} holds whenever $\Delta\lambda_{i+1}\leq e^{-\eta\, i}$ with $\eta \geq -\ln{(C_{\tilde{\rho}}\,\tilde{\rho})}$. 
%Finally, assuming that $\frac{(\tilde{\rho}-\rho^k)}{\rho^k}\leq \frac{(\delta + \gamma)}{\varepsilon_0}$ and defining $C\coloneqq \frac{(\tilde{\rho}-\rho^k)}{\rho^k}\frac{\varepsilon_0}{(\delta+\gamma)} $, we obtain that Inequality~\eqref{eq:ineq_6} holds whenever $\Delta\lambda_i\leq e^{-\eta\, i}$ with $\eta \geq -\ln{(C\,\tilde{\rho})}$.
\end{proof}
\end{document}